\documentclass{article}

\usepackage[preprint]{neurips_2024}

\usepackage[utf8]{inputenc} 
\usepackage[T1]{fontenc}    
\usepackage{hyperref}       
\usepackage{url}            
\usepackage{booktabs}       
\usepackage{amsfonts}       
\usepackage{nicefrac}       
\usepackage{microtype}      
\usepackage{xcolor}         

\usepackage{algorithm}
\usepackage{algorithmic}
\usepackage{tikz}
\usepackage{makecell}
\usepackage{amsthm}
\usepackage{hyperref}
\usepackage{mathtools}
\usepackage{amsmath}
\usepackage{amsthm}
\usepackage{amssymb}
\usepackage{setspace}
\usepackage{booktabs}
\usepackage{tabularx}
\usepackage{bm}
\usepackage{xcolor}
\definecolor{magenta}{rgb}{1.0, 0.0, 1.0}

\newcommand{\junpei}[1]{}
\newcommand{\updated}[1]{}
\newcommand{\kazu}[1]{}
\newcommand{\brown}[1]{}
\newcommand{\ks}[1]{}

\makeatletter

\newcommand{\freqepsilon}{\Delta}

\usepackage{amsthm}
\usepackage{amsfonts}
\usepackage{bbm}
\usepackage{enumerate}
\usepackage{float}
\usepackage{caption}

\usepackage{subfigure}

\newlength{\subfigcolsep}
\theoremstyle{remark}

\theoremstyle{definition}
\newtheorem{thm}{Theorem}
\newtheorem{lem}[thm]{Lemma}
\newtheorem{definition}{Definition}

\newtheorem{cor}[thm]{Corollary}

\newtheorem{remark}{Remark}
\newtheorem{example}{Example}
\newtheorem{claim}{Claim}

\usepackage{amsmath} 
\usepackage{amssymb}
\usepackage{bm}
\usepackage{ascmac}
\usepackage{setspace}
\usepackage[at]{easylist}
\usepackage{comment}

\usepackage{url}

\usepackage{breakurl}

\usepackage{amsmath}
\DeclareMathOperator*{\argmax}{arg\,max}

\newcommand{\Real}{\mathbb{R}}

\newcommand{\Ex}{\mathbb{E}}
\newcommand{\Prob}{\mathbb{P}}

\newcommand{\Ind}{\bm{1}}

\newcommand{\EA}{\mathcal{A}}

\newcommand{\EC}{\mathcal{C}}

\newcommand{\EE}{\mathcal{E}}
\newcommand{\EF}{\mathcal{F}}

\newcommand{\EH}{\mathcal{H}}

\newcommand{\EN}{\mathcal{N}}

\newcommand{\ET}{\mathcal{T}}

\newcommand{\EX}{\mathcal{X}}

\newcommand{\bmu}{\bm{\mu}}
\newcommand{\bnu}{\bm{\nu}}
\newcommand{\blambda}{\bm{\lambda}}

\newcommand{\hatmu}{\hat{\mu}}

\newcommand{\bH}{\bm{H}}

\newcommand{\bP}{\bm{P}}

\newcommand{\PoE}[1]{\Prob_{#1}[J \ne i^*(#1)]}

\newcommand*\diff{\mathop{}\!\mathrm{d}}

\newcommand{\PoEPair}{\mathrm{PoE}}

\newcommand{\Deltatarget}{\Delta_0}
\newcommand{\deltaone}{\delta^{(1)}}
\newcommand{\deltatwo}{\delta^{2}}
\newcommand{\Conf}{\mathrm{Conf}}
\newcommand{\rmUCB}{\mathrm{UCB}}
\newcommand{\rmLCB}{\mathrm{LCB}}

\newcommand{\paren}[1]{\mathopen{}\left( {#1}_{{}_{}}\,\negthickspace\right)\mathclose{}}
\newcommand{\bracket}[1]{\mathopen{}\left[ {#1}_{{}_{}}\,\negthickspace\right]\mathclose{}}
\newcommand{\ceil}[1]{\mathopen{}\lceil {#1}_{{}_{}}\,\negthickspace\rceil\mathclose{}}

\allowdisplaybreaks

\makeatother

\title{Fixed Confidence Best Arm Identification in the Bayesian Setting}

\author{%
  Kyoungseok Jang\\
  Department of Computer Science\\
  Universit\'a degli Studi di Milano\\
  Milan, MI 20133, Italy\\
  \texttt{ksajks@gmail.com} \\
  \And  
  Junpei Komiyama\\
  Stern School of Business\\ 
  New York University\\
  New York, NY 10012, United States \\
  \texttt{junpei.komiyama@stern.nyu.edu} \\
  \And
  Kazutoshi Yamazaki \\
  The School of Mathematics and Physics\\
  The University of Queensland\\
  St Lucia, Brisbane QLD 4072, Australia \\
  \texttt{k.yamazaki@uq.edu.au} \\
}

\begin{document}

\maketitle

\begin{abstract}

We consider the fixed-confidence best arm identification (FC-BAI) problem in the Bayesian setting. This problem aims to find the arm of the largest mean with a fixed confidence level when the bandit model has been sampled from the known prior. 
Most studies on the FC-BAI problem have been conducted in the frequentist setting, where the bandit model is predetermined before the game starts. 
We show that the traditional FC-BAI algorithms studied in the frequentist setting, such as track-and-stop and top-two algorithms, result in arbitrarily suboptimal performances in the Bayesian setting. 
We also obtain a lower bound of the expected number of samples in the Bayesian setting and introduce a variant of successive elimination that has a matching performance with the lower bound up to a logarithmic factor. Simulations verify the theoretical results.
\end{abstract}

\section{Introduction}

In many sequential decision-making problems, the learner repeatedly chooses an arm (option) to play with and observes a reward drawn from the unknown distribution of the corresponding arm. 
{One of the most widely-studied instances of such problems is the multi-armed bandit problem \citep{Thompson1933, Robbins1952, lai1987}, where the goal is to maximize the sum of rewards during the rounds. 
Since the learner does not know the distribution of rewards, they need to explore the different arms, and yet, exploit the arms of the most rewarding arms so far.}
{Different from the classical bandit formulation, there are situations where one is more interested in collecting information rather than maximizing intermediate rewards. 
The best arm identification (BAI) is a sequential decision-making problem in which 
the learner is only interested in identifying the arm with the highest mean reward. While the origin of this problem goes back to at least the 1950s \citep{bechhofer1954single, Paulson1964,gupta1977}, recent work in the field of machine learning reformulated the problem \citep{Audibert10}.
}{In the BAI, the learner needs to pull arms efficiently for better identification. To achieve efficiency and accuracy, the learner should determine which arm to choose based on the history, when to stop the sampling, and which arm to recommend as the learner's final decision. 
}

There are two types of BAI problems depending on the optimization objective.
In the fixed-budget (FB) setting \citep{Audibert10}, the learner attempts to minimize the probability of error (misidentification of the best arm) given a limited number of arm pulls $T$. In the fixed confidence (FC) setting \citep{DBLP:conf/ciss/JamiesonN14}, the learner attempts to minimize the number of arm pulls, subject to a predefined probability of error $\delta \in (0,1)$. In this paper, we shall focus on the FC setting, {which is useful when we desire a rigorous statistical guarantee.}

Most of the previous BAI studies focus on the frequentist setting, where the bandit model is chosen adversarially from some hypothesis class beforehand. 
In this setting, several algorithms, such as Track and Stop \citep{kaufman16a} and Top-two algorithms \citep{Russo2016,qin2017improving,jourdan2022top}, are widely known. 
These algorithms have an optimal sample complexity, meaning that they are one of the most sample-efficient algorithms among the class of $\delta$-correct algorithms. 



The sample complexity of these algorithms is problem-dependent. To see this, consider the following example.
\begin{example}{\rm (A/B/C testing)}
Consider A/B/C testing of web designs. We have three arms (web designs) from which we would like to find the largest retention rate via allocating users to web designs $i=1,2,3$. If we attempt to find the best arm with confidence $\delta$, we may need a large number of samples (users) when the suboptimality gap (the gap between the retention rate of the best arm and the second best arm) is small because in such a case the identification of the best arm is difficult {-- 
the minimum number of samples required is inversely proportional to the square of the suboptimality gap. For example, when comparing the testing of retention rates of (0.9, 0.5, 0.1) with (0.9, 0.89, 0.1), the second case requires around $\paren{\frac{0.9-0.5}{0.9-0.89}}^2 = 1600$ times more samples compared to the first case.} 
\end{example}
{In practice, the retention rate of $0.89$ in the second case may be an acceptably good enough compared to the optimal retention rate of $0.9$}, and we may stop exploration at the moment the learner identifies a reasonably good arm, which is the first or the second arm in this example. This idea is formalized in several ways. The literature of Ranking and Selection (R\&S) usually considers the indifference-zone formulation \citep{hong2021review}. In the context of best arm identification, a similar notion of $\epsilon$-best answer identification has also been considered \citep{DBLP:conf/nips/MaronM93,EvanDar2006,GabillonGL12,DBLP:conf/colt/KaufmannK13,jourdan2023varepsilon}. In these settings, the learner accepts a sub-optimal arm whose means are at most $\epsilon$ worse than the mean of the optimal arm. 
Other related settings include the good arm identification problem \citep{DBLP:journals/ml/KanoHSMNS19,DBLP:journals/ml/TabataNHK20,DBLP:conf/icml/ZhaoSSJ23}, where the goal is to identify an arm that exceeds the predefined threshold, and the thresholding bandit problem \citep{DBLP:conf/icml/LocatelliGC16,DBLP:journals/corr/abs-1910-06368}, where the goal is to identify whether each arm is above or below the threshold. {All these problem settings require an extra parameter, like $\epsilon$ or an acceptance threshold, that directly determines the acceptance level. Even though the algorithm's performance is influenced by this parameter, it is often challenging to determine a reasonable value for it in advance.}

In this paper, we study an alternative approach based on the Bayesian setting.
In particular, we consider the prior distribution on the model parameters. We relax the requirement on the correctness of the best arm identification by using the prior belief. Rather than requiring the frequentist $\delta$-correctness for any model, we require the learner to have marginalized correctness over the prior distribution, which we call Bayesian $\delta$-correctness. 

We study {the fixed confidence BAI (FC-BAI) problem in Bayesian setting}. 
Our contributions are as follows.
\begin{itemize}
    \item \textbf{First,} we find that in the Bayesian setting, the performance of the traditional frequentist setting-based algorithms, such as Track and Stop and Top-two algorithms, can be arbitrarily worse (Section \ref{sec:preliminary}). This is because frequentist approaches spend too many resources when the suboptimality gap is narrow. 
    \item \textbf{Second, } we prove that the lower bound of the number of expected samples should attain at least the order of $\Omega(\frac{L(\bH)^2 }{\delta} )$ {as $\delta \to 0$} (Section \ref{sec:lower bound}). {Here $L(\bH)$ is our novel quantity that represents the sample complexity with respect to the prior distribution $\bH$. 
    }
    This order is different from the existing lower bound in the frequentist setting\footnote{In fact, marginalizing the frequentist sample complexity over the prior distribution leads to an unbounded value.}, implying that the Bayesian setting is essentially different from the frequentist setting. 
    \item \textbf{Third, }we design an algorithm whose expected sample size is upper-bounded by 
    $O(\frac{L(\bH)^2 }{\delta} \log \frac{L(\bH)}{\delta})$ (Section \ref{sec:main algorithm}). {Our algorithm is based on the elimination algorithm \citep{DBLP:conf/nips/MaronM93,EvanDar2006,doi:10.1287/opre.2014.1282}, but we add an early stopping criterion to prevent over-commitment of the algorithm for a bandit model with a narrow suboptimality gap. 
    Our algorithm has a matching upper bound up to {the logarithmic factor}.}
\end{itemize}

{We also conduct simulation to demonstrate that the sample complexity of frequentist algorithms does indeed diverge in a bandit model with a small suboptimality gap, even in very simple cases (Section \ref{sec:simulation}).}

\subsection{Related work}
To our knowledge, BAI problems studied for the Bayesian setting have been limited to the fixed budget setting \citep{komiyamakkc_bayesian,atsidakou2023bayesian}. 
\citet{komiyamakkc_bayesian} showed that, in the fixed-budget setting, a simple non-Bayesian algorithm has an optimal simple regret up to a constant factor, implying that the advantage the learner could get from the prior is small when the budget is large. 
This is very different from our fixed-confidence setting, where utilizing the prior distribution is necessary.

Several FC-BAI algorithms used Bayesian ideas on the structure of the algorithm, {although} most of those studies used frequentist settings for measuring the guarantee. The `Top-Two' type of algorithms are the leading representatives in this direction. The first instance of top-two algorithms, which is called Top-Two Thompson sampling (TTTS), is introduced in the context of Bayesian best arm identification. TTTS requires a prior distribution, and  \citet{Russo2016} showed that the sample complexity of posterior convergence of TTTS which is the same as the sample complexity of the frequentist fixed-confidence best arm identification. 
Subsequent research analyzed the performance of TTTS in the frequentists' viewpoint \citep{shang2020fixed}. 
Later on, the idea of top-two sampling is then extended into many other algorithms, such as Top-Two Transportation Cost \citep{shang2020fixed}, Top-Two Expected Improvement (TTEI, \citealt{qin2017improving}), Top-Two Upper Confidence bound (TTUCB, \citealt{DBLP:journals/corr/abs-2210-05431}).
Even though some of the top two algorithms adapt a prior, they implicitly solve the optimization that is justified in view of frequentist.



{Another line of Bayesian sequential decision-making is Bayesian optimization \citep{DBLP:conf/icml/SrinivasKKS10,mockus2012bayesian,DBLP:journals/pieee/ShahriariSWAF16,DBLP:conf/aistats/JamiesonT16,frazier_tutorial_bo}, where the goal is to find the best arm in Bayesian setting. Note that Bayesian optimization tends to deal with structured identification, especially for Gaussian processes, and most of the algorithms for Bayesian optimization do not have specific stopping criteria. 
}



    \section{Problem setup}\label{sec:problem setup}


We study the fixed confidence best arm identification problem (FC-BAI) in a Bayesian setting. In this setup, we have $k$ arms {in the set} $[k]:=\{1,2,\dots,k\}$ with \textit{unknown} distribution $\bP=(P_1, \cdots, P_k)$ which is drawn from a \textit{known} prior distribution at time 0, namely $\bH=(H_1, \cdots, H_k)$. {The unknown bandit model} $P_i$ is a one-parameter distribution, and $\bP$ is specified by $\bmu:= (\mu_1, \cdots, \mu_k)$. To simplify the problem, we will focus on the Gaussian case, where each $P_i$ is a Gaussian distribution with known variance $\sigma_i^2$. Each mean of $P_i$, {denoted} $\mu_i$, is drawn from a known prior Gaussian distribution $H_i$, which can be written as $N(m_i, \xi_i^2)$.  

At every time step $t=1,2,\cdots$, the forecaster chooses an arm $A_t \in [k]$ and observes a reward $X_t$, which is drawn independently from $P_{A_t}$. Since we focus on the Gaussian case, $X_t\sim N(\mu_{A_t}, \sigma_{A_t}^2)$ conditionally given $A_t$ {and $\mu_{A_t}$}. After each sampling, the forecaster {must} decide whether to continue the sampling process or stop sampling and make a decision $J \in [k]$.


Let $\mathcal{F}_t = \sigma(A_1, X_1, A_2, X_2, \cdots, A_t, X_t)$ be the $\sigma$-field generated by observations up to time $t$. 
The algorithm of the forecaster $\pi:= ((A_t)_t, \tau, J)$ is defined by the following triplet \citep{kaufman16a}:
\begin{itemize}
\item A sampling rule $(A_t)_t$, which determines the arm to draw at round $t$ based on the previous history (each $A_t$ {must} be $\mathcal{F}_{t-1}$ measurable).
\item A stopping rule $\tau$, which means when to stop the sampling (i.e., stopping time with respect to $\mathcal{F}_t$). 
\item A decision rule $J$, which determines the arm the forecaster recommends based on his sampling history (i.e., $J$ is $\mathcal{F}_\tau$-measurable).
\end{itemize}

In FC-BAI, the forecaster aims to recommend arm $J$ that correctly identify (one of) the best arm(s) $i^*(\bmu):= \argmax_{i\in[k]} \mu_i$ with probability at least $1-\delta$. 
Since the case of multiple best arms is of measure zero under $\bH$, we can focus on $\bmu$ such that $i^*(\bmu)$ is unique.
For the FC-BAI problem in the Bayesian setting, we use the \textit{expected} probability of misidentification:
\begin{equation}
    \mathrm{PoE} (\pi;\bH) := \Ex_{\bmu \sim \bH} \bracket{\Prob\paren{J\neq i^*{(\bmu)} | \EH_{\bmu}}}, \label{eqn: expected PoE}
\end{equation} 
where 
$\mathcal{H}_{\bmu} := \{ \textrm{$\bmu$ is the correct bandit model} \}.
$
Now we formally define the algorithm of interest as follows: 
\begin{definition}{\rm (Bayesian $\delta$-correctness)} 
For a prior distribution $\bH$, an algorithm $\pi=((A_t), \tau, J)$ is {said to be }Bayesian $(\bH, \delta)$-correct if it satisfies $\PoEPair (\pi;\bH) \le \delta$. Let $\mathcal{A}^b(\delta, \bH)$ be the set of Bayesian $(\bH, \delta)$-correct algorithms for the prior distribution $\bH$.
\end{definition}
The objective of the FC-BAI problem in the Bayesian setting is to find an algorithm $\pi =((A_t)_t,\tau, J)\in \EA^b(\delta, \bH)$ that minimizes $\Ex_{\bmu \sim \bH}[\tau]$.
\paragraph{Terminology} Define $N_i (t)= \sum_{s=1}^{t-1}  \Ind[A_s = i]$ as the number of times arm $i$ is pulled before timestep $t$. Let $h_i$ be the probability density function of $H_i$. {Since we consider Gaussian prior, $h_i(\mu_i) := (1/\sqrt{2\pi\xi_i})\exp(- (\mu_i-m_i)^2/(2\xi_i^2))$.}
Let $i^*, j^*: \mathbb{R}^k \to [k]$ {be the best and the second best arm under the input such that for each $\bmu \in \{\mathbf{x} \in \mathbb{R}^k: x_i \neq x_j \forall i \neq j \}$,} $i^*(\bmu) = \argmax_{i \in [K]} \mu_i$ and $j^*(\bmu) = \argmax_{i \in [K]\backslash \{i^*(\bmu)\}} \mu_i$. 

Let $\mathrm{KL}_i (a \Vert b):= \frac{(a-b)^2}{2\sigma_i^2}$ represent {the} KL-divergence between two Gaussian distributions with equal variances (the variance of the $i$-th arm $\sigma_i^2$) but different means, denoted as $a$ and $b$. Similarly, {$d(a,b) := a \log(a/b) + (1-a) \log((1-a)/(1-b))$} is the KL divergence between two Bernoulli distributions with means $a$ and $b$. Throughout this paper, $\Ex_{\bmu}$ and $\Prob_{\bmu}$ denote the expectation and probability when the bandit model is fixed as $\bmu \in \mathbb{R}^k$, i.e., $\Ex_{\bmu} = \Ex[\cdot|\EH_{\bmu}]$ and $\Prob_{\bmu} = \Prob(\cdot|\EH_{\bmu})$. We will abuse the notation $\PoEPair$ so that for $\blambda \in \mathbb{R}^k$, $\PoEPair (\pi;\blambda)$ means 
\begin{equation*}
    \mathrm{PoE} (\pi;\blambda) := \Prob_{\blambda} \paren{J\neq i^* (\blambda) | \EH_{\blambda}}.
\end{equation*} 
Naturally, $\PoEPair (\pi;\bH) = \Ex_{\bmu \sim \bH} \bracket{\PoEPair (\pi;\bmu)}$.


Lastly, we introduce the constant $L(\bH)$ which characterizes the sample complexity of the FC-BAI problem in a Bayesian setting. 

\begin{definition}\label{def:Lij}
    For each $i, j \in [k]$, define $L(\bH)$ as follows:
    $$L(\bH):= \sum_{i,j \in [k], i\neq j} L_{ij} (\bH) \text{ where }L_{ij}(\bH) := \int_{-\infty}^{\infty} h_i (x) h_j (x) \prod_{s: s \in [k]\backslash \{i,j\}} H_s (x) \diff x.$$
\end{definition}

This constant has the following interesting property which we call a volume lemma:
{
\begin{lem}[Volume Lemma, informal]\label{lem:volume_lem}
For $\Delta \in (0,1)$, let
    \begin{align*}
        L(\bH, \Delta):=\frac{1}{\Delta} \Prob_{\bmu \sim \bH} \bracket{\mu_{i^* (\bmu)}-\mu_{j^* (\bmu)}\leq \Delta}.
    \end{align*}
Then, $\lim_{\Delta \to 0^+} L({\bH},\Delta) = L(\bH)$. Especially, for $\Delta<\frac{L(H)}{\sum_{i \in [k]}\frac{2(k-1)}{\xi_i}}$, $L(H,\Delta)\in (\frac{1}{2}L(H), 2L(H))$.
\end{lem}
}
The volume lemma states that the volume of prior where the suboptimality gap is smaller than $\Delta$ is proportional to $L(\bH) \Delta$ when $\Delta$ is small. The quantity $L_{ij}(\bH) \Delta$ is the probability where {arms} $i$ and $j$ are the $\Delta$-close top-two arms. 
The formal version of this lemma, which involves some regularity conditions, is shown in Appendix \ref{appsec:volumelemma}.

\section{Limitation of traditional frequentist approaches in the Bayesian setting} \label{sec:preliminary}


Existing BAI studies mainly focused on the Frequentist $\delta$-correct algorithms which are defined as follows:
\begin{definition}[Frequentist $\delta$-correctness]
An algorithm $\pi=((A_t), \tau, J)$ is {said to be} frequentist $\delta$-correct if, for any bandit instance $\bmu \in \mathbb{R}^k$ {such that $i^*(\bmu)$ is unique}, it satisfies $\PoEPair (\pi;\bmu) \le \delta$. Let $\mathcal{A}^f(\delta)$ be {the} set of {all} frequentist-$\delta$-correct algorithms.
\end{definition}
For the frequentist $\delta$-correct algorithms, there is a known lower bound for the expected stopping time as follows: for all bandit instance $\bmu {\in \mathbb{R}^k}$ and for all $((A_t), \tau, J) \in \EA^f (\delta)$,
\begin{equation}\label{eqn: Freq FC-BAI lower bound} 
\Ex_{\bmu}\left[
\tau
\right] \geq \log(\delta^{-1}) T^*(\bmu) + o(\log(\delta^{-1}))
\end{equation}
where $T^* (\bmu)$ is a sample complexity function {dependent} on the bandit instance $\bmu$.\footnote{ For details about $T^* (\bmu)$, a reader may refer to \citet{garivier16}.} Moreover, many of the known Frequentist $\delta$-correct algorithms achieve asymptotic optimality \citep{garivier16, Russo2016, tabata23, Qin2017}, meaning that they are orderwisely tight up to the lower bound on Eq. \eqref{eqn: Freq FC-BAI lower bound} as $\delta \to 0$. 
However, little is known, or at least discussed, about their performance in the Bayesian setting. 

One can check that a frequentist $\delta$-correct algorithm is also Bayesian $\delta$-correct as well {($\mathcal{A}^f(\delta) \subset \mathcal{A}^b(\delta, \bH)$ for all $\bH$)}. Naturally, our interest is whether {or not} the most efficient classes of frequentist $\delta$-correct algorithms, such as Tracking algorithms and Top-two algorithms, are efficient in Bayesian settings.
Somewhat surprisingly, the following theorem states that any $\delta$-correct algorithm is suboptimal in Bayesian settings. 
\begin{thm}\label{thm_ttts_inf}
For all $\delta > 0, \bH$ and $((A_t), \tau, J) \in \mathcal{A}^f(\delta),$ $\Ex_{\bmu \sim \bH}\left[
\tau
\right]
= +\infty.$
\end{thm}

{Proof of Theorem \ref{thm_ttts_inf} is} found in Appendix \ref{appthm: preliminary proof}. 
{To illustrate the proof, we will use a two-armed Gaussian instance as an example.}

\subsection{Special case - two armed Gaussian case}\label{subsec: two armed gaussian}

In this subsection, we consider the case $k = 2$.
Here we present one intuitive corollary of the lower bound theorem \citep{kaufman16a, garivier16, kaufmann2016complexity} that uses a standard information-theoretic technique. 

\begin{cor}[\citealt{kaufmann2014complexity}]\label{cor:two arm gaussian lb} Let $\delta\in(0, 1)$. For any {frequentist} $\delta$-correct algorithm $((A_t), \tau, J)$ and for any fixed mean vector $\bmu=(\mu_1, \mu_2) \in \mathbb{R}^2$,
    $\Ex_{\bmu} [\tau] \geq \frac{d(\delta, 1-\delta)}{(\mu_1 -\mu_2)^2} \geq \frac{\log \frac{1}{2.4 \delta}}{(\mu_1 -\mu_2)^2}$.
\end{cor}

{In the frequentist setting, Corollary \ref{cor:two arm gaussian lb} implies the lower bound of $\Ex_{\bmu_0} [\tau] = \Omega(\log(\delta^{-1})/(\mu_1-\mu_2)^2)$, which is $\Omega(\log(\delta^{-1}))$ when we view parameters $(\mu_1, \mu_2)$ as constants.
However, in the Bayesian setting, the algorithm is given the prior distribution $\bH$ on $\bmu$, and thus the stopping time is marginalized over $\bH$. In particular, limiting our interest to the case of $|\mu_1 - \mu_2|<\freqepsilon$, 
we can obtain the following lower bound:
\begin{align*}
    \Ex_{\bmu \sim \bH} [\tau] \geq \Ex_{\bmu \sim \bH} [\tau \cdot  \Ind[|\mu_1 -\mu_2| \leq \freqepsilon]]
    \geq \frac{\log \delta^{-1}}{\freqepsilon^2} \Prob_{\bmu \sim \bH} [|\mu_1- \mu_2| \le \freqepsilon] = \Omega\left(\frac{\log \delta^{-1}}{\freqepsilon}\right),
\end{align*}
where we use the volume lemma (Lemma \ref{lem:volume_lem}) in the last transformation.
This inequality implies that if we naively use a known frequentist $\delta$-correct algorithm in the Bayesian setting, the expected stopping time will diverge because we can choose an arbitrarily small $\freqepsilon$. 
The case of a small gap is \textit{difficult to identify}, and the expected stopping time can be very large for such a case if we aim to identify the best arm for any model.
}

\section{Lower bound}\label{sec:lower bound}

This section will elaborate on
the lower bound of the stopping time in the Bayesian setting. 
Theorem \ref{thm: main lower bound} {below} states that any Bayesian $(\bH, \delta)$-correct algorithm requires the expected stopping time of at least $\Omega(\frac{L(\bH)^2}{\delta})$. 
\begin{thm}\label{thm: main lower bound}
    {Define $\sigma_{\min} = \min_{i \in [k]} \sigma_i^2$ and $N_V = \frac{ L(\bH)^2 \sigma_{\min}^2 \ln 2}{16e^4 \delta}$.}
    Let $\delta < \delta_L (\bH)$ be sufficiently small.\footnote{
    In particular, $\delta_L(\bH)$ is defined in Appendix \ref{appsec:delta0}.} Then, for any BAI algorithms $\pi=((A_t), \tau, J)$, if $ \Ex_{\bmu \sim \bH} [\tau] \leq N_V$, then $\mathrm{PoE}(\pi;\bH) \geq \delta$. 
\end{thm}
In this main body, we will use the two-armed Gaussian bandit model with homogeneous variance condition (i.e. $\sigma_1 = \sigma_2 = \sigma$) for easier demonstration of the proof sketch.
{Theorem \ref{thm: main lower bound}, which is more general in the sense that it can deal with $k>2$  arms with heterogeneous variances, is proven in Appendix \ref{appsec: lower bound}.} 

\paragraph{Sketch of the proof, {for $k=2$:}} It suffices to show that the following is an empty set:
\[
\mathcal{A}^b (\delta, \bH, N_V) := \{ \pi \in \mathcal{A}^b (\delta, \bH): \Ex_{\bmu \sim \bH} [\tau] \leq N_V\}.
\]

Assume that $\mathcal{A}^b (\delta, \bH, N_V) \neq \emptyset$ and choose an arbitrary $\pi \in \EA^b (\delta, \bH, N_V)$. We start from the following transportation lemma:

\begin{lem}[\citealt{kaufman16a}, Lemma 1]\label{lem_transportation_lemma} Let $\delta \in (0, 1)$. For any algorithm $((A_t), \tau, J)$, any $\EF_\tau$-measurable event $\EE$, any bandit models $\bmu, \blambda \in \{(x,y)\in \mathbb{R}^2: x\neq y\}$ such that $i^*(\bmu) \neq i^*(\blambda)$, 
\begin{equation*}
    \Ex_{\bmu} \bracket{\sum_{i=1}^2 \mathrm{KL}_i(\mu_i, \lambda_i) N_i (\tau)} \geq d(\Prob_{\bmu} ({\EE}), \Prob_{\blambda} ({\EE})).
\end{equation*}
\end{lem}
Note that the above Lemma holds for any algorithm, and thus works for any stopping time $\tau$. Now define $\bnu (\bmu)$ as a swapped version of $\bmu \in \mathbb{R}^2$, which means $(\bnu(\bmu))_1 =\mu_2, \bnu(\bmu)_2 = \mu_1$, and let $\EE = \{J \neq i^* (\bmu)\}$, the event that the recommendation of the algorithm is wrong. Substituting $\blambda$ with $\bnu(\bmu)$ from the above equation of Lemma \ref{lem_transportation_lemma} leads 
\begin{align}\label{eqn:modified_transportation_lemma}
    \Ex_{\bmu} \bracket{\frac{(\mu_1-\mu_2)^2}{2\sigma^2} \tau} &\geq d(\PoEPair (\pi;\bmu),1-\PoEPair (\pi;\bnu)) \geq \log \frac{2}{2.4 (\PoEPair (\pi;\bmu)+\PoEPair(\pi;\bnu))}. 
\end{align}
Note that the first inequality comes from the fact that $\EE$, the failure event of the bandit model $\bmu$, is exactly a success event of $\bnu(\bmu)$ in this two-armed case, and the last inequality is from our modified lemma (Lemma \ref{lem: k-error-2}) from Eq. (3) of \cite{kaufman16a}. One can rewrite the above inequality as
\begin{equation}\label{eqn:modified_transportation_lemma}
    \frac{\PoEPair (\pi;\bmu)+\PoEPair(\pi;\bnu)}{2} \geq \frac{1}{2.4}\exp\paren{\Ex_{\bmu} \bracket{- \frac{(\mu_1 - \mu_2)^2}{2\sigma^2} \tau}}.
\end{equation}

We can rewrite the conditions of $\mathcal{A}^b (\delta, \bH, N_V)$ as
\begin{align*}
    \PoEPair (\pi;\bH)&=\int_{\bmu \in \mathbb{R}^2} \PoEPair (\pi;\bmu) \diff \bH (\bmu) \leq \delta \qquad
    \text{and } & \int_{\bmu \in \mathbb{R}^2} \Ex_{\bmu} \bracket{\tau} \diff \bH (\bmu) \leq N_V. \tag{Opt0}
\end{align*}
Using Eq. \eqref{eqn:modified_transportation_lemma} {and with some symmetry tricks}, we get $V_0 \leq \PoEPair (\pi;\bH)$ where
\begin{align}\label{eqn:strategy condition}
    V_0 := \int_{\bmu \in \mathbb{R}^2} \frac{1}{2.4} \exp {\paren{-\frac{(\mu_1 -\mu_2)^2}{2\sigma^2}\Ex_{\bmu} \bracket{\tau}}} \diff \bH (\bmu) \leq \delta 
    \ \text{ and } \int_{\bmu \in \mathbb{R}^2} \Ex_{\bmu} \bracket{\tau} \diff \bH (\bmu) \leq N_V. \tag{Opt1}
\end{align}
Note that on the above two inequalities, only $\Ex_{\bmu}[\tau]$ is the value that depends on the algorithm $\pi$. Now our main idea is that we can relax these two inequalities to the following optimization problem by substituting $\Ex_{\bmu}[\tau]$ to an arbitrary $\tilde{n}: \mathbb{R}^2 \to [0,\infty)$ as follows:
\begin{align}\label{eqn:opt1}
    V := \inf_{\tilde{n}:\mathbb{R}^2\to [0,\infty)} \int_{\bmu \in \mathbb{R}^2} \exp {\paren{-\frac{(\mu_1 -\mu_2)^2}{2\sigma^2}\tilde{n}(\bmu)}} \diff \bH (\bmu)\ 
    \text{ s.t. } \int_{\bmu \in \mathbb{R}^2} \tilde{n}(\bmu) \diff \bH (\bmu) \leq N_V \tag{Opt2}
\end{align}
and $V\leq 2.4 V_0 \leq 2.4\delta$. {Let $\EN:= \{ \bmu \in \mathbb{R}^2: |\mu_1 - \mu_2|<\Delta:=\frac{8\delta}{L(\bH)}\}$. From the discussions in Section \ref{subsec: two armed gaussian}, one might notice that $\EN$ is an important region for bounding $\Ex_{\bmu \sim \bH}[\tau]$.} We can relax the above \eqref{eqn:opt1} to the following version, which focuses more on $\EN$:
\begin{align}\label{eqn:opt2}
    V' := \inf_{\tilde{n}:\mathbb{R}^2\to [0,\infty)} \int_{\bmu \in \EN} \exp {\paren{-\frac{(\mu_1 -\mu_2)^2}{2\sigma^2}\tilde{n}(\bmu)}} \diff \bH (\bmu) 
    \text{ s.t. } \int_{\bmu \in \EN} \tilde{n}(\bmu) \diff \bH (\bmu) \leq N_V. \tag{Opt3}
\end{align}
One can prove $V'\leq V \leq 2.4 \delta$. Now, if we notice that function $x \mapsto \exp {\paren{-\frac{(\mu_1 -\mu_2)^2}{2\sigma^2}x}}$ is a convex function, we can use Jensen's inequality to verify that the optimal solution for \eqref{eqn:opt2} is when $\tilde{n} = \frac{N_V}{\Ex_{\bmu \sim \bH} \bracket{\Ind_\EN}} \Ind_{\EN}$, and when we use $N_V$ in Theorem \ref{thm: main lower bound}, one can get: 
\begin{align}
{V'}&\ge \int_{\bmu \in \EN}
\exp\left(-\frac{\Delta^2}{2\sigma^2}\cdot \paren{\frac{N_V}{\EE_{\bmu \sim \bH}[\Ind_{\EN}]}} \right) 
\diff \bH(\bmu) \tag{by optimality of $\tilde{n} = \frac{N_V}{\Ex_{\bmu \sim \bH} \bracket{\Ind_\mathcal{N}}}$} \\
&\geq \int_{\bmu \in \EN}
\exp\left(-\frac{\Delta^2 \cdot N_V }{\sigma^2 \Delta L(\bH)}\right) 
\diff \bH(\bmu) \geq \exp\left(-\frac{\Delta \cdot N_V }{\sigma^2 L(\bH)}\right) \cdot \Delta L(\bH) \tag{both by Lemma \ref{lem:volume_lem}}\\
&> 2.4\delta \tag{by definition of $N_V$ and $\Delta$}, 
\label{ineq:poe_as_draw}
\end{align}
which is a contradiction. 
This means no algorithm satisfies \eqref{eqn:strategy condition}, and the proof is completed. 
\section{Main algorithm} \label{sec:main algorithm}

\begin{algorithm}[t]
\caption{Successive Elimination with Early-Stopping} 
\label{alg:elim}
\begin{algorithmic}
\STATE {\bfseries Input:}{ Confidence level $\delta$, prior $\bH$}
\STATE $\Deltatarget:=\frac{\delta}{4L(\bH)}$
\STATE Initialize the candidate of best arms $\EA (1) = [K].$ 
\STATE $t=1$
\WHILE {True}
\STATE Draw each arm in $\EA(t)$
once. $t  \leftarrow t+|\EA(t)|$.\label{line_draw}

\FOR{$i \in \EA(t)$}
\STATE Compute $\rmUCB(i,t)$ and $\rmLCB(i,t)$ from \eqref{eqn:UCB LCB bound}.
\IF{$\rmUCB(i,t) \le \max_j \rmLCB(j,t)$} 
\STATE $\EA(t) \leftarrow \EA(t) \setminus \{i\}$.
\ENDIF   
\ENDFOR
\IF{$|\EA(t)|=1$}
\STATE {\bfseries Return} arm $J$ in $\EA(t)$.\label{line_stopone} 
\ENDIF
\STATE Compute $\hat{\Delta}^{\mathrm{safe}}(t):= \max_{i \in \EA(t)} \rmUCB(i,t) - \max_{i\in \EA(t)} \rmLCB(i,t)$. 
\IF{$\hat{\Delta}^{\mathrm{safe}}(t) \le \Deltatarget$}
\STATE {\bfseries Return} arm $J$ which is uniformly sampled from $\EA(t)$.\label{line_stoptwo} 
\ENDIF
\ENDWHILE
\end{algorithmic}
\end{algorithm}

This section introduces our main algorithm (Algorithm \ref{alg:elim}). In short, our algorithm is a modification of the elimination algorithm with the incorporation of the indifference zone technique. 
Before the sampling starts, define $\Deltatarget:= \frac{\delta}{4L(\bH)}$ which satisfies the following condition, thanks to Lemma \ref{lem:volume_lem}:
$$ \Prob_{\bmu \sim \bH} (\mu_{i^* (\bmu)} - \mu_{j^*(\bmu)} \leq \Deltatarget) \leq \frac{\delta}{2}.$$
In each iteration {of the \textbf{while} loop of Algorithm \ref{alg:elim}}, the learner selects and observes each arm in the active set. After inspecting all arms, the algorithm calculates the confidence bounds for each arm in the active set using the formula as follows: let $\Conf(i,t)$ and $\hatmu_i (t)$ be the confidence width and the empirical mean of arm $i$ at time $t$ as 
\begin{align*}
\Conf(i, t) := \sqrt{2\sigma_i^2 \frac{\log(6(N_i(t))^2/((\frac{\delta^2}{2K})\pi^2))}{N_i(t)}}, \qquad \hatmu_i (t) := \sum_{s=1}^{t-1} X_s \Ind[A_s = i].
\end{align*}
Then the upper and lower confidence bounds of arm $i$ at timestep $t$, denoted as $\mathrm{UCB}$ and $\mathrm{LCB}$ respectively, can be defined in the following manner:
\begin{align}\label{eqn:UCB LCB bound}%
\mathrm{UCB}(i,t):=\hatmu_i(t) + \Conf(i, t), \qquad
\mathrm{LCB}(i,t) :=\hatmu_i(t) - \Conf(i, t).
\end{align}

After calculating $\mathrm{UCB}$ and $\mathrm{LCB}$, the algorithm eliminates arms with $\mathrm{UCB}$ smaller than the largest $\mathrm{LCB}$ and maintains only arms that could be optimal in the active set $\EA$. Up to this point, it follows the traditional elimination 
approach.

The main difference in our algorithm lies in the stopping criterion. At the end of each iteration, the algorithm checks the stopping criterion. Unlike typical elimination algorithms that continue until only one arm remains, we have introduced an additional indifference condition. This condition arises when the suboptimality gap is so small that identifying them would require an excessive number of samples. In such cases, our algorithm stops additional attempts to identify differences between arms in the active set and randomly recommends one from the active set instead.

\begin{remark}
    In the context of PAC-($\epsilon$, $\delta$) identification, \citet[Remark 9]{EvanDar2006} introduced a similar approach. The largest difference is that they use the parameter $\epsilon$ as a parameter that defines the indifference-zone level, whereas our parameter $\Delta_0$ is spontaneously derived from the prior $\bH$ and the confidence level $\delta$ without specifying {the} indifference-zone. 
\end{remark}

Theorem \ref{thm: elim upperbound} describes the theoretical guarantee of the Algorithm \ref{alg:elim}. 

\begin{thm}\label{thm: elim upperbound}
    For $\delta < 4L(\bH)\cdot \min\paren{\frac{L(\bH)}{{\sum_{i\in[k]}\frac{k-1}{\xi_i}}}, \paren{\min_{i,j\in[k]} \xi_i L_{ij}(\bH)}^2}$, Algorithm \ref{alg:elim} which consists of $((A_t), \tau, J)$ has the expected stopping time upper bound as follows:
    \begin{equation}
        \Ex_{\bmu \sim \bH}[\tau] \leq C \cdot \sigma_{\max}^2 \frac{L(\bH)^2}{\delta} \log\left( \frac{L(\bH)}{\delta}\right)
        + O(\log \delta^{-1}),
    \end{equation}
    where $C=320 \paren{\frac{\pi^2}{3}+1}$ is a universal constant and $\sigma_{\max} = \max_{i\in[k]}\sigma_i$. Here, $O(\log \delta^{-1})$ is a function of $\delta$ and $\bH$ that is proportional to $\log \delta^{-1}$ when we view prior parameters $\bH$ as constants. 
    Plus, the strategy defined by Algorithm \ref{alg:elim} {is in $\EA^b(\delta, \bH)$.} 
\end{thm}

See Appendix \ref{appsec: upper bound proof} for the formal proof of Theorem \ref{thm: elim upperbound}.

\begin{remark}
    When we compare the lower bound (Theorem \ref{thm: main lower bound}) with the upper bound of Algorithm \ref{alg:elim} (Theorem \ref{thm: elim upperbound}), we can see the algorithm is near-optimal. 
    If we view $\sigma_{\max} / \sigma_{\min}$ as a constant, the bounds are tight up to a $\log \frac{L(\bH)}{\delta}$ factor.
\end{remark}
{
\begin{remark}
    The condition $\delta<4L(\bH)\cdot \min\paren{{L(\bH)/ {\sum_{i\in[k]}\frac{k-1}{\xi_i}}}, \min_{i,j\in[k]} (\xi_i L_{ij}(\bH))^2}$ is only for cleaner illustration of the regret bound in Theorem \ref{thm: elim upperbound}. The non-asymptotic result, when $\delta$ is a moderately large constant, can be found in Appendix \ref{appsec: nonasymptotic}. 
\end{remark}
}
\paragraph{Proof sketch of Theorem \ref{thm: elim upperbound}} 

We summarize the general strategy for the proof as follows. By the law of total expectation, $\Ex_{\bmu \sim \bH} [\tau] = \Ex_{\bmu \sim \bH} \bracket{\Ex_{\bmu}[\tau]}$. Therefore, we will first derive a frequentist upper bound of $\Ex_{\bmu} [\tau]$, and then marginarize it to obtain the Bayesian expected stopping time.

First, with the confidence bound defined as Eq. \eqref{eqn:UCB LCB bound} we have the following guarantee that the true means for all arms are in the confidence bound interval with high probability. 
\begin{lem}\label{lem_exp_subopt_main} For any fixed $\bmu \in \{\mathbf{v} \in \mathbb{R}^k: v_i \neq v_j \text{ for all }i, j \in [k]\}$, let $\EX(\bmu):= \{\forall i\in[k]\text{ and }t\in \mathbb{N}, \ \mu_i \in (\mathrm{LCB}(i,t), \mathrm{UCB}(i,t))\}.$
Then, $\Prob_{\bmu} \bracket{\EX(\bmu)} \geq 1-\delta^2$.
\end{lem}
Now we can rewrite $\Ex_{\bmu \sim \bH}\left[
\tau
\right]
$ as follows:
\begin{align}
\Ex_{\bmu \sim \bH}\left[
\tau
\right]
=&
\Ex_{\bmu \sim \bH}\left[\Ex_{\bmu}[
\tau]
\right]\tag{Law of Total Expectation}\\
=& \Ex_{\bmu \sim \bH}\left[\Ex_{\bmu}[
\tau\Ind[\EX(\bmu)]]
\right]+\Ex_{\bmu \sim \bH}\left[\Ex_{\bmu}[
\tau\Ind[\EX(\bmu)^c]]
\right]\nonumber\\
=&
\sum_i
\Ex_{\bmu \sim \bH}\bracket{
\Ex_{\bmu}[N_i(\tau) \Ind[\EX(\bmu)]]}+
\Ex_{\bmu \sim \bH}\left[\Ex_{\bmu}[
\tau\Ind[\EX(\bmu)^c]]
\right].\label{eqn: intermediate expected stopping time}
\end{align}
Let $\Delta_i = \Delta_i (\bmu) := (\max_{s \in [k]} \mu_s)- \mu_i$ and $R_0(\Delta):\approx \lceil C \sigma_{\max}^2 \cdot \frac{\log \Delta^{-1}}{\Delta^2} \rceil$.
{
For the first term, under $\EX(\bmu)$, we can bound $N_i(\tau) $ by $ R_0(\max(\Delta_0, \Delta_i))$ (Lemma \ref{lem_exp_subopt} in Appendix \ref{appsec: upper bound proof}), and integrate it over the prior distribution obtain the leading factor.
For the second term, thanks to the indifference stopping condition ($\hat{\Delta}^{\mathrm{safe}}(t) \le \Deltatarget$), one can prove that $\tau$ is always smaller than $R(\Delta_0)$ (Lemma \ref{lem_panic} in Appendix \ref{appsec: upper bound proof}), which leads non-leading term.
}

To check that the expected probability of error is below $\delta$, we have an additional lemma:

\begin{lem}[Probability of dropping $i^*(\bmu)$]\label{lem_wrongdrop_main}
For any $\bmu_0$, under $\EH_{\bmu_0}$, $
\EX(\bmu_0) \subset \bigcap_t \left\{
i^*(\bmu_0) \in \EA(t)
\right\}.$
\end{lem}
This lemma means under the event $\EX (\bmu)$, the best arm is never dropped. We can also prove that under the event $\EX(\bmu)$, if $\Delta_i(\bmu) > \Deltatarget$, the sub-optimal arm will eventually be dropped before the algorithm terminates (Lemma \ref{lem_panic} in Appendix \ref{appsec: upper bound proof}). These two facts mean there are only two cases in which the prediction of Algorithm \ref{alg:elim} could be wrong.
\begin{itemize}
    \item Under $\EX(\bmu)^c$, both facts cannot guarantee the  correct identification. From Lemma \ref{lem_exp_subopt_main}, {$\Prob_{\bmu}[\EX(\bmu)^c] \leq \delta^2$ for all $\bmu$, and thus $\Prob_{\bmu \sim \bH} [\EX(\bmu)^c] \leq \delta^2$}. 
    \item When $\Delta_i(\bmu) \leq \Deltatarget$. From Lemma \ref{lem:volume_lem} and the definition of $\Deltatarget$, the probability of drawing such $\bmu$ from the prior is at most $\delta/2$.
\end{itemize}
Therefore, by union bound, Algorithm \ref{alg:elim} has the expected probability of misidentification guarantee smaller than $\delta^2 + \delta/2 < \delta$. 
\section{Simulation}\label{sec:simulation}




We conduct two experiments to demonstrate that the expected stopping times of frequentist $\delta$-correct algorithms diverge in a Bayesian setting and that the elimination process in Algorithm \ref{alg:elim} is necessary for more efficient sampling. For Tables \ref{table:simulation result} and \ref{table:without elim result}, each column `Avg', `Max', and `Error' represents the average stopping time, maximum stopping time, and the ratio of the misidentification, respectively.\footnote{Due to space limitations, we include the computation time in the Appendix \ref{appsubsec:comp time}} More details of these experiments are in Appendix \ref{appsec:experiment details}.

\paragraph{Frequentist algorithms diverge in Bayesian Setting} We evaluate the empirical performance of our Elimination algorithm (Algorithm \ref{alg:elim}) by comparing it with other frequentist algorithms such as Top-two Thompson Sampling (TTTS) \citep{Russo2016} and Top-two UCB (TTUCB) \citep{jourdan2022non}. 

 We design an experiment setup that has $k=2$ arms with standard Gaussian prior distribution, which means $m_i =0, \xi_i=1$ for all $i \in [k]$. We set $\delta=0.1$ and ran $N=1000$ Bayesian FC-BAI simulations to estimate the expected stopping time and success rate. 

\begin{table}[t]
\parbox{.45\textwidth}{
\caption{Comparison of two top-two algorithms and Algorithm \ref{alg:elim}.}
\label{table:simulation result}
\begin{center}
\begin{small}
\begin{sc}
\begin{tabular}{c|ccc}
\toprule
& Avg& Max & Error \\
\midrule
Alg. \ref{alg:elim}    &  $1.06 \times 10^4$& $2.35\times 10^5$ & 1.5\% \\
TTTS & $1.56\times 10^5$ & $1.09\times 10^8$ & 0.5\% \\
TTUCB & $1.95\times 10^5$ & $1.13\times 10^8$ & 0\% \\
\bottomrule
\end{tabular}
\end{sc}
\end{small}
\end{center}
}\hspace*{\fill}
\parbox{.45\textwidth}{
\caption{{Comparison of Algorithm \ref{alg:elim} and the no-elimination version of it.} }
\label{table:without elim result}
\begin{center}
\begin{small}
\begin{sc}
\begin{tabular}{c|ccc}
\toprule
& Avg& Max & Error \\
\midrule
Alg. \ref{alg:elim}  &  $2.69\times 10^5$& $1.66\times 10^7$ & 0.6\% \\
NoElim & $1.29\times 10^6$ & $8.25\times 10^7$ & 0\%\\
\bottomrule
\end{tabular}
\end{sc}
\end{small}
\end{center}
}
\vskip -0.1in
\end{table}


{In Table \ref{table:simulation result}, one can see that the two top-two algorithms exhibit has very large maximum stopping time. This supports our theoretical result in Section \ref{sec:preliminary} that the expected stopping time of Frequentist $\delta$-correct algorithms will diverge in the Bayesian setting. {We did not check the track and stop algorithm \citep{garivier16} because it needs to solve an optimization for each round, but the fact that the expected stopping time of the track and stop is at least half of the TTTS and TTUCB for a small $\delta$ implies that the performance of track and stop is similar to that of top-two algorithms.}. Algorithm \ref{alg:elim} shows a significantly smaller average stopping time as well as an average computation time than that of these algorithms. }

\paragraph{Effect of the elimination process} 
We implemented the modification of Algorithm \ref{alg:elim} (denote as NoElim) {that never eliminates an arm from $\EA(t)$}
\footnote{See Appendix \ref{appsubsec:noelim} for the pseudocode.}
In this setup, we have $k=10$ arms with standard Gaussian prior distribution, which means $m_i=0, \xi_i=1$ for all $i\in [k]$. We set $\delta=0.01$ and ran $N=1000$ Bayesian FC-BAI simulations.

As one can check from Table \ref{table:without elim result}, elimination of arms helps the efficient use of samples and reduces stopping time and computation time.
\section{Discussion and future works}


{We have considered the {Gaussian} Bayesian best arm identification with fixed confidence. We show that the traditional Frequentist FC-BAI algorithms do not stop in finite time in expectation, which implies the suboptimality of such algorithms in the Bayesian FC-BAI problem. 
We have established a lower bound of the Bayesian expected stopping time, which is of order $\Omega(\frac{L(\bH)^2}{\delta})$.
{Moreover, we have introduced} the elimination and early stopping algorithm, {which achieves a matching stopping time up to a polylogarithmic factor of $L(\bH)$ and $\delta$.}
We conduct simulations to support our results. }

{In the future, we will attempt to tighten the logarithmic and $\paren{\frac{\max_i \sigma_i}{\min_i \sigma_i}}^2$ gap between the lower and upper bound, extend the indifference zone strategy for other traditional BAI algorithms in Bayesian setting, extend our analysis from Gaussian bandit instances to general exponential families, and design a robust algorithm against misspecified priors. }

\bibliography{main}
\bibliographystyle{plainnat}

\appendix
\newpage

\section{Notation table \label{sec_notation}}
\newcommand{\DeltaThr}{\Delta_{thr}}

\begin{table}[h]
\begin{center}
\caption{Major notation}
\label{tbl_not}
\renewcommand{\arraystretch}{1.1} 
\begin{tabular}{lll} 
symbol & definition 
\\ \hline
$k$ & number of the arms \\
$\delta$ & confidence level \\
$\bmu$ & means $(= (\mu_1,\mu_2,\dots,\mu_k))$ \\
$\bH$ & prior distribution of $\bmu$ \\
$H_i$ & prior distribution of $\mu_i$ \\
$h_i$ & prior density of $\mu_i$ \\
$m_i, \xi_i$ & mean and standard deviation of $h_i$ \\
$N_i(t)$ & $\sum_{s=1}^{t-1}  \Ind[A_s = i]$ \\
$L(\bH)$ & See Definition \ref{def:Lij}\\
$i^*(\bmu), j^*(\bmu)$ & best arm and second best arm\\
$\bnu(\bmu)$ & alternative model where the top-two means of $\bmu$ are swapped \\
$\mathrm{KL}_i(\cdot, \cdot)$ & KL divergence between two distributions\\
$d(p, q)$ & KL divergence between two Bernoulli distributions with parameters $p$ and $q$\\
$N_i(t)$ & number of draws on arm $i$ before time step $t$\\
$B$ & $320 \max_{i \in [k]} \sigma_i^2$\\
$B_0$ & $\paren{\frac{\pi^2}{3}+1}B$\\
$\Theta_{i}$ & $\{ \bmu \in \mathbb{R}^k: i^*(\bmu)=i\}$\\
$\Theta_{ij}$ & $\{ \bmu \in \mathbb{R}^k: i^*(\bmu)=i, j^*(\bmu)=j\}$\\
$\bmu_{\backslash i}$ ( $\bmu_{\backslash i, j}$) & the vector projection which omits $i$-th coordinate ($i,j$-th, respectively)\\
$\bH_{\backslash i}$ ( $\bH_{\backslash i, j}$) & the distribution which omits $i$-th coordinate ($i,j$-th, respectively)
\end{tabular}
\end{center}
\end{table}

\begin{table}[h]
\begin{center}
\caption{Notations for the lower bound proof, Section \ref{appsec: lower bound}}
\label{tbl_not_lb}
\renewcommand{\arraystretch}{1.1} 
\begin{tabular}{lll} 
symbol & definition 
\\ \hline
$R$ & $e^{-4}$ \\
$\tilde{\Delta}$ & $\frac{32e^4}{L(\bH)}\delta$\\
$\bnu(\bmu)$ & alternative model where the top-two means are swapped \\
$n_i(\bmu)$ & $\Ex_{\bmu}[N_i(\tau)]$\\
$D_0(\bH)$ & $\begin{cases}
        W(-\frac{1}{32 \max_{i\in [k]} \xi_i^{3/2}})& \text{If $\max_{i\in[k]} \xi_i > \sqrt[3]{\frac{e^2}{2^{10}}}$} \\
        1 & \text{Otherwise}
    \end{cases}$ ($W$ is the Lambert W function.)
    \\
$D_1 (\bH)$ & $\min_{i \neq j} \bracket{\left|\frac{m_i}{\sigma_i^2} -\frac{m_j}{\sigma_j^2}\right|^{-1}, \bracket{\frac{1}{2\sigma_i^2}+\frac{1}{2\sigma_j^2}}^{-2}}$\\
$\delta_L(\bH)$ & $\frac{L(\bH)}{32e^4} \cdot \min \paren{D_0(\bH), D_1(\bH), \min_{i \in [k]}\frac{1}{4m_i^2},\frac{L(\bH)}{4(k-1)\sum_{i \in [k]}\frac{1}{\xi_i}}}$
\end{tabular}
\end{center}
\end{table}

\begin{table}[h]
\begin{center}
\caption{Notations for the upper bound proof, Section \ref{appsec: upper bound proof}}
\label{tbl_not_ub}
\renewcommand{\arraystretch}{1.1} 
\begin{tabular}{lll} 
symbol & definition 
\\ \hline
$\Conf(i, t)$ & confidence bound (Section \ref{sec:main algorithm})\\
$\rmUCB(i,t), \rmLCB(i,t)$ &
upper and lower confidence bounds (Section \ref{sec:main algorithm})\\
$\hat{\Delta}^{\mathrm{safe}}(t)$ 
&
$\max_{i \in \EA(t)} \rmUCB(i,t) - \max_{i\in \EA(t)} \rmLCB(i,t)$\\
$\Delta_0$ & $\frac{\delta}{4L(\bH)}$\\
$\DeltaThr$ & $\min \paren{\log \frac{4\sqrt{k}}{\delta \pi}, \frac{1}{B}}$\\
$R_0(\Delta)$ & $B \frac{\log \min(\Delta, \DeltaThr)^{-1}}{\min(\Delta, \DeltaThr)^2}$\\
$ T_0 $ & $k R_0(\Deltatarget)$\\
$\Delta_s (\bmu)$ & $\mu_{i^*(\bmu)} - \mu_s$\\
$\mathcal{X} (\bmu)$ & Event $\bigcap_{i\in[k]} \left[\paren{\bigcap_{t=1}^\infty\left\{\rmLCB(i,t) \le \mu_i \right\}} {\bigcap} \paren{\bigcap_{t=1}^{\infty}\left\{\rmUCB(i,t) \ge \mu_i \right\}}\right].
$ (Eq.~\eqref{ineq_goodevent})\\
\end{tabular}
\end{center}
\end{table}

\section{Proof of Lemma \ref{lem:volume_lem}}\label{appsec:volumelemma}

We will use the following formal version of the volume lemma for the proof. 
The first result is used for the upper bound, and the second result is used for the lower bound.

\newcommand{\LijHdel}[1]{L_{ij}({\bH}, #1)}
\newcommand{\LijHdelp}[1]{L_{ij}'({\bH}, #1)}
\newcommand{\LHdel}[1]{L({\bH}, #1)}
\newcommand{\LHdelp}[1]{L'({\bH}, #1)}
{
\begin{lem}[Volume Lemma, formal]\label{lem:volume_rewrite} Let $\Theta_{ij}:= \{ \bmu \in \mathbb{R}^k: i^*(\bmu)=i, j^*(\bmu)=j\}$. 
\begin{enumerate}
    \item For any $\Delta \in (0,1)$, define $$\LijHdel{\Delta}:= \frac{1}{\delta}\int_{\Theta_{ij}} \Ind[|\mu_i- \mu_j| \le \Delta] \diff\bH(\bmu).$$ Then, 
    \begin{equation}\label{appeqn: volume lemma}
\LijHdel{\Delta} \in \bracket{L_{ij}(\bH)- \frac{1}{\xi_i}\Delta , L_{ij}(\bH)+ \frac{1}{\xi_i} \Delta}.
    \end{equation}
    Especially when $\Delta < \frac{L(\bH)}{{\sum_{i\in[k]}\frac{k-1}{\xi_i}}}$, $L(\bH, \Delta) \leq 2 L(\bH)$. 
    \item (Volume lemma for the lower bound) For small enough positive real number $\Delta<\min\bracket{\frac{1}{4\max_{i\in[k]} m_i^2}, D_0 (\bH)}
    $\footnote{See Appendix \ref{appsec:delta0} for the definition of $D_0 (\bH)$.} let $$\LijHdelp{\Delta}:=    \frac{1}{\Delta} \int_{\Theta_{ij}} \Ind[|\mu_i- \mu_j| \le \Delta] \Ind[|\mu_i|, |\mu_j|\leq \frac{1}{\sqrt{\Delta}}]  \diff \bH(\bmu).$$ Then, 
\begin{equation}\label{appeqn: volume lemma lower}
\LijHdelp{\Delta} \in \bracket{L_{ij}(\bH)-\frac{2}{\xi_i}\Delta, L_{ij}(\bH)+\frac{1}{\xi_i}\Delta}.
\end{equation}
Especially when $\Delta < \min\bracket{\frac{1}{4\max_{i\in[k]} m_i^2}, D_0 (\bH), \frac{L(\bH)}{{\sum_{i\in[k]}\frac{4(k-1)}{\xi_i}}}}$, $L'(\bH, \Delta)\in \bracket{\frac{1}{2}L(\bH), 2L(\bH)}$.
\end{enumerate}
    \end{lem} 
}

\begin{proof}
{First, let us prove the upper bound of Eq. \eqref{appeqn: volume lemma}, {i.e.} $L_{ij}(\bH,\delta) \leq L_{ij}(\bH)+\frac{1}{\xi_i}\Delta$.}

Recall for $i,j \in [k]$,
$\Theta_{ij} = \{ \mu: i^* (\mu)=i, j^* (\mu)=j\}$. We have
\begin{align*}
\int_{\Theta_{ij}} \Ind\left[|\mu_i-\mu_{j}| \le \Delta
\right] \diff \bH(\bmu) &= \int_{-\infty}^\infty  \int_{\mu_j}^{\mu_j + \Delta}  h_i(\mu_i) \diff \mu_i \prod_{k \neq i,j} \int_{-\infty}^{\mu_j} h_k (\mu_k) \diff  \mu_k h_j(\mu_j) \diff  \mu_j \\
 &\leq \int_{-\infty}^\infty  \left(\Delta \max_{0 \leq y \leq \Delta}  h_i(\mu_j+y) \right)  \prod_{k \neq i,j} \int_{-\infty}^{\mu_j} h_k (\mu_k) \diff  \mu_k h_j(\mu_j) \diff  \mu_j \\
  &\leq \Delta \int_{-\infty}^\infty  \bracket{h_i(\mu_{j}) + \frac{e^{-1/2}}{\xi_i}\Delta} \prod_{k \neq i,j} \int_{-\infty}^{\mu_j} h_k (\mu_k) \diff  \mu_k h_j(\mu_j) \diff  \mu_j \tag{by {the} Lipschitz property of {the} Gaussian {density}, $e^{-1/2}/\xi_i$ is the steepest slope of $N(m_i, \xi_i^2)$}\\
    &\leq \Delta \left( \underbrace{\int_{-\infty}^\infty  h_i(\mu_{j})  \prod_{k \neq i,j} \int_{-\infty}^{\mu_j} h_k (\mu_k) \diff  \mu_k h_j(\mu_j) \diff  \mu_j}_{= L_{ij} (\bH)} + \bracket{\frac{e^{-1/2}}{\xi_i}\Delta}  \right).
\end{align*}

Therefore, we verified the upper bound side of Eq. \eqref{appeqn: volume lemma}.

For the lower bound, by {following the same steps,}
one can prove $\int_{\Theta_{ij}} \Ind\left[|\mu_i-\mu_{j}| \le \Delta
\right] \diff \bH(\bmu) \geq (L_{ij} (\bH) - \frac{1}{\xi_i} \Delta) \Delta$. 

For the proof of Eq. \eqref{appeqn: volume lemma lower}, by Chernoff's method we can bound the tail probability as follows:
$$ \int_{\Theta_{ij}} \Ind\left[|\mu_i - m_i|> \frac{1}{\sqrt{\Delta}} \right] \diff H(\bmu) <{\int_{\mathbb{R}^k}}\Ind \left[|\mu_i - m_i|> \frac{1}{\sqrt{\Delta}} \right] \diff H(\bmu) < 2\exp \paren{-\frac{1}{2\Delta \xi_i^2}}.$$
When $|m_i| < \frac{1}{2\sqrt{\Delta}}$, then {we can} change the above inequality as:
\begin{align*}
    \int_{\Theta_{ij}} \Ind\left[|\mu_i|> \frac{1}{\sqrt{\Delta}} \right] \diff H(\bmu) &\leq \int_{{\mathbb{R}^k}} \Ind\left[|\mu_i|> \frac{1}{\sqrt{\Delta}} \right] \diff H(\bmu) \\
    &\leq \int_{{\mathbb{R}^k}} \Ind\left[|\mu_i - m_i|> \frac{1}{{2}\sqrt{\Delta}} \right] \diff H(\bmu) 
    < 2\exp \paren{-\frac{1}{8\Delta \xi_i^2}}.
\end{align*}

Therefore, 
\begin{align*}
    \int_{\Theta_{ij}} \Ind[|\mu_i-\mu_j| \le \Delta] &\Ind[|\mu_i|, |\mu_j|\leq \frac{1}{\sqrt{\Delta}}]  \diff \bH(\bmu) \geq \int_{\Theta_{ij}} \Ind\left[|\mu_i-\mu_{j}| \le \Delta\right] \diff \bH(\bmu) \\
    &-\int_{\Theta_{ij}} \Ind\left[|\mu_i - m_i|> \frac{1}{2\sqrt{\Delta}} \right] \diff H(\bmu) -\int_{\Theta_{ij}} \Ind\left[|\mu_j - m_j|> \frac{1}{2\sqrt{\Delta}} \right] \diff H(\bmu) \\&\geq  (L_i (\bH) - \frac{1}{\xi_i} \Delta) \Delta - 2\exp \paren{-\frac{1}{8\Delta \xi_i^2}}-2\exp \paren{-\frac{1}{8\Delta \xi_j^2}}\\
    &\geq (L_i (\bH) - \frac{2}{\xi_i} \Delta) \Delta \tag{by $\Delta < D_0(\bH)$}
\end{align*}
for $\Delta< \min \bracket{ \min_{i \in [k]} \frac{1}{4m_i^2},D_0(\bH)}$. 
\end{proof}

\section{Proof of Theorem \ref{thm_ttts_inf}}\label{appthm: preliminary proof}

{For $\freqepsilon > 0$,} let $\Theta_i (\freqepsilon):=\{\bmu \in \mathbb{R}^k: i^*(\bmu)=i, \mu_i - \mu_{j^*(\bmu)}\leq \freqepsilon \} $. From Lemma \ref{lem:volume_rewrite}, we have
\begin{align*}
    \Prob_{\bmu \sim \bH} \bracket{i^*(\bmu)=i, \mu_i - \mu_{j^*(\bmu)}\leq \freqepsilon}&= \sum_{j\neq i} \Prob_{\bmu \sim \bH} \bracket{i^*(\bmu)=i, j^*(\bmu)=j, \mu_i - \mu_j \leq \freqepsilon}\\
    &\geq \sum_{j\neq i} \frac{1}{2} L_{ij}(\bH) \freqepsilon.
\end{align*}

Now, for each $\bmu\in \mathbb{R}^k$, let $\nu: \mathbb{R}^k \to \mathbb{R}^k$ be a function such that for $s\in[k]$,
\begin{align*}
  \nu(\bmu)_s:= \begin{cases}
    \mu_{i^*(\bmu)} & \text{when $s=j^*(\bmu)$}\\
    \mu_{j^*(\bmu)} & \text{when $s=i^*(\bmu)$}\\
    \mu_{s} & \text{Otherwise}.
    \end{cases}
\end{align*}
For any $\bmu \in \mathbb{R}^k$, let $\mathcal{E}({\bmu}) = \{ J \neq i^*({\bmu}) \}$ and $\bnu = \nu (\bmu)$. By Lemma 1 in \cite{kaufman16a}, for any frequentist $\delta$-correct algorithm $\pi=((A_t)_t, \tau, J)$, we have

\begin{equation}\label{appeqn:intermediate transportation eqn}
\sum_{s\in [k]} \Ex_{\bmu}[N_s(\tau)]\mathrm{KL}_s(\mu_s|| \nu_s)
\ge d(P_{\bmu} (\mathcal{E}({\bmu})), P_{\bnu}(\mathcal{E}({\bmu}))), \quad {\bmu \in \mathbb{R}^k.}
\end{equation}

For the left side, from the construction of $\bnu$, 

\begin{align*}
    \mathrm{KL}_s(\mu_s || \nu_s):=\begin{cases}
        \frac{(\mu_{i^*(\bmu)} -\mu_{j^*(\bmu)})^2 }{2\sigma_s^2} & \text{$s=i^*(\bmu), j^*(\bmu)$}\\
        0 & \text{Otherwise}.
    \end{cases}
\end{align*}

Since $\pi \in \EA^f(\delta)$ and from the definition of $\EE (\bmu)$ and $\bnu$, $\Prob_{\bmu} (\mathcal{E}({\bmu}))\leq \delta$ and $\Prob_{\bnu}(\mathcal{E}({\bmu})) \geq 1-\delta$. 

Overall, we can rewrite Eq. \eqref{appeqn:intermediate transportation eqn} to the following simpler form:

\begin{align*}
\Ex_{\bmu}[N_{i^*(\bmu)}(\tau)] \frac{(\mu_{i^*(\bmu)}-\mu_{j^*(\bmu)})^2}{2\sigma_{i^*(\bmu)}^2}
+\Ex_{\bmu}[N_{j^*(\bmu)}(\tau)] \frac{(\mu_{i^*(\bmu)}-\mu_{j^*(\bmu)})^2}{2\sigma_{j^*(\bmu)}^2}
\ge d(\delta, 1-\delta)\\
\Longrightarrow \Ex_{\bmu}[N_{i^*(\bmu)}(\tau) + N_{j^*(\bmu)}(\tau)] \geq \frac{2d(\delta,1-\delta) \min_{s\in[k]} \sigma_s^2}{(\mu_{i^*(\bmu)}-\mu_{j^*(\bmu)})^2}.
\end{align*}

Since $\tau = \sum_{s=1}^k N_s (\tau)$, we can lower bound the expected stopping time when $\bmu$ is given as follows:
\begin{align}\label{appeqn: lower bound on conditional expected stopping time}
    \Ex_{\bmu}[\tau]&\geq \Ex_{\bmu}[N_{i^*(\bmu)}(\tau) + N_{j^*(\bmu)}(\tau)]\geq \frac{2d(\delta,1-\delta) \min_{s\in[k]} \sigma_s^2}{(\mu_{i^*(\bmu)}-\mu_{j^*(\bmu)})^2}. 
\end{align}

Now, when we compute marginal $\Ex_{\bmu}[\tau]$ over $\bmu$, we have
\begin{align*}
    \Ex_{\bmu \sim \bH} [\tau] &= \Ex_{\bmu \sim \bH}\bracket{\Ex_{\bmu}[\tau]} \tag{Law of total expectation}\\
    &\geq \Ex_{\bmu \sim \bH}\bracket{\Ex_{\bmu}[\tau] \Ind_{\Theta_i (\freqepsilon)}}\\
    &\geq \Ex_{\bmu \sim \bH} \bracket{ \frac{2d(\delta,1-\delta) \min_{s\in[k]} \sigma_s^2}{(\mu_{i^*(\bmu)}-\mu_{j^*(\bmu)})^2}\Ind_{\Theta_i (\freqepsilon)}} \tag{Eq. \eqref{appeqn: lower bound on conditional expected stopping time}}\\
    &\geq \Ex_{\bmu \sim \bH} \bracket{ \frac{2d(\delta,1-\delta) \min_{s\in[k]} \sigma_s^2}{\freqepsilon^2}\Ind_{\Theta_i (\freqepsilon)}} \\
    &\geq \frac{\paren{\sum_{j\neq i} L_{ij}(\bH)}d(\delta,1-\delta) \min_{s\in[k]} \sigma_s^2}{\freqepsilon^2} = \frac{\paren{\sum_{j\neq i} L_{ij}(\bH)}d(\delta,1-\delta) \min_{s\in[k]} \sigma_s^2}{\freqepsilon}.
\end{align*}
Now since $\freqepsilon$ is an arbitrary small positive number, we can conclude that $\Ex_{\bmu \sim \bH} [\tau]$ diverges.

\section{Proof of Theorem \ref{thm: main lower bound}} \label{appsec: lower bound}

\newcommand{\rmKL}{\mathrm{KL}}

In this subsection, we will prove the following theorem:
\begin{thm}[Restatement of Theorem \ref{thm: main lower bound}] \label{lem:k_asym}

Let $\delta > 0$ be sufficiently small such that $\delta <\delta_L (\bH)$.
    For any best arm identification algorithm, {if}
    $$ \int_{{\mathbb{R}^k}} \paren{\sum_{i\in[k]} n_i (\bmu)} \diff  \bH(\bmu) \leq N_V,$$
    {then}
    $$\int_{{\mathbb{R}^k}} \PoE{\bmu} \diff \bH(\bmu) \geq \delta.$$
\end{thm}

\begin{proof}

\newcommand{\ubmu}{\bmu}
\newcommand{\ubnu}{{\bnu}}
\newcommand{\umu}{{\mu}}
\newcommand{\unu}{{\nu}}
By Lemma 1 in \cite{kaufman16a}, for any stopping time $\tau$, we have

\begin{equation}\label{ineq_poepair_biased_raw}
\sum_{i \in [k]} \Ex_{\ubmu}[N_i(\tau)]\rmKL_i(\umu_i || \unu_i)
\ge d (\Prob_{\ubmu} (\mathcal{E}({\ubmu})), \Prob_{\ubnu}(\mathcal{E}({\ubmu}))), \quad {\ubmu \in \mathbb{R}^k.}
\end{equation}

To modify the RHS of Eq. \eqref{ineq_poepair_biased_raw}, we will use the following lemma:
\begin{lem}\label{lem: k-error-2}
    For any $p, q', q \in (0,1)$ such that $q'\leq q$, we have 
    $$\ln \frac{1}{4 (p+q)} \leq d(p, 1-q').$$
\end{lem}
By using Lemma \ref{lem: k-error-2} and the fact that $\Prob_{\ubnu}(\mathcal{E}({\ubmu})) \geq 1- \PoE{\ubnu}$ by definition, we can transform \eqref{ineq_poepair_biased_raw} into
\begin{equation}\label{ineq_poepair_biased_karms}
\sum_{i  \in [k]} \Ex_{\ubmu}[N_i(\tau)]\rmKL_i(\umu_i || \unu_i)
\ge \ln \frac{1}{4(\PoE{\ubmu}+\PoE{\ubnu})}, \quad \forall \ubmu \in \mathbb{R}^k.
\end{equation}

Note that $\Prob_{\bmu}(\mathcal{E} {(\bmu)})$ is exactly the error probability, and we are interested in the marginal error probability $\Ex_{\bmu \sim H} [\Prob_{\bmu}(\mathcal{E})]$. 
Let $n_i(\bmu) = \Ex_{\bmu}[N_i(\tau)]$, and $\tilde\Delta$ is an arbitrary small enough positive variable which we will define later on Eq. \eqref{appeqn:tildedelta}. Then, 
by rearrangement, we can induce the following inequalities
\begin{align*}
\label{ineq_mpoe_lower_asym_karms}
\lefteqn{
\int_{\mathbb{R}^k} \PoE{\bmu} \diff \bH(\bmu) 
}\\
&=  \sum_{i\in [k]} \sum_{j\neq i} \int_{ \Theta_{ij}} \mathbb{P}_{\bmu} [J \neq i]  \diff \bH (\bmu)\\
&=  \frac{1}{2}\paren{\sum_{i\in [k]} \sum_{j\neq i} \int_{ \Theta_{ij}} \mathbb{P}_{\bmu} [J \neq i]  \diff \bH (\bmu)+\sum_{j\in [k]} \sum_{i\neq j} \int_{ \Theta_{ji}} \mathbb{P}_{\bnu (\bmu)} [J \neq j]  h_i (\mu_j) h_j (\mu_i) \paren{\prod_{s\neq i,j} h_s (\mu_s)}\diff \bmu} \tag{Symmetry of the Lebesgue measure}\\
&=  \frac 1 2 \left( \sum_{i\in [k]} \sum_{j\neq i} \int_{ \Theta_{ij}} \bracket{\mathbb{P}_{\bmu} [J \neq i] h_i (\mu_i) h_j (\mu_j)+ \mathbb{P}_{\bnu (\bmu)} [J \neq j]  h_i (\mu_j) h_j (\mu_i)} \paren{\prod_{s\neq i,j} h_s (\mu_s)}\diff \bmu\right)\\
&\geq 
\sum_{i\in [k]}\sum_{j\neq i} \int_{\Theta_{ij}} \frac{\PoE{\bmu} + \PoE{\bnu}}{2}  \min \left( h_i (\mu_i ) h_j (\mu_j), h_i (\mu_j ) h_j (\mu_i) \right) \diff \bmu \tag{$AC+BD \geq A \min (C,D) + B \min (C,D) = (A+B) \min(C,D)$ for $A,B,C,D>0$}
\\ 
&\geq 
\sum_{i\in [k]}\sum_{j\neq i} \int_{\Theta_{ij}} \frac{\exp\left(-n_i(\bmu)\rmKL_i(\mu_i , \mu_j ) - n_j (\bmu) \rmKL_j (\mu_j, \mu_i)\right)}{8} \min \left( 1, \frac {h_i (\mu_j ) h_j (\mu_i)} {h_i (\mu_i ) h_j (\mu_j)} \right) h_i (\mu_i ) h_j (\mu_j) \diff \bmu \tag{Eq. \eqref{ineq_poepair_biased_karms}}\\
&\geq e^{-4} \sum_{i\in [k]}\sum_{j\neq i} \int_{\Theta_{ij}}
\frac{\exp\left(-n_i(\bmu)\rmKL_i(\mu_i , \mu_j ) - n_j (\bmu) \rmKL_j (\mu_j, \mu_i)\right)}{8} \Ind\bracket{|\mu_i|,|\mu_j|\leq \frac{1}{\sqrt{\tilde\Delta}}}
\diff H(\bmu).  \tag{Lemma \ref{lem:ratiolemma}}
\end{align*}

For the last inequality, we used the following lemma. {The proof for this lemma is found in}  Subsection \ref{subsec_ratiolemma}. 
\begin{lem}[Ratio Lemma]\label{lem:ratiolemma}
    For all $a,b \in \mathbb{R}$ which satisfy $|a -b | \leq D_1(\bH)$ and $|a|, |b| \leq\frac{1}{\sqrt{D_1(\bH)}}$ for some fixed $D_1(\bH)$\footnote{$D_1 (\bH):= \min_{i \neq j} \bracket{\left|\frac{m_i}{\sigma_i^2} -\frac{m_j}{\sigma_j^2}\right|^{-1}, \bracket{\frac{1}{2\sigma_i^2}+\frac{1}{2\sigma_j^2}}^{-2}}$}, $\frac{h_i (a) h_j (b)}{h_i (b) h_j (a)} \geq e^{-4}$ for all $i, j \in [k]$. 
    \end{lem}


In short, we have 
\begin{align}
    \int_{\mathbb{R}^k} &\PoE{\bmu} \diff \bH(\bmu) \nonumber \\
    &\geq e^{-4} \sum_{i\in [k]}\sum_{j\neq i} \int_{\Theta_{ij}}
\frac{\exp\left(-n_i(\bmu)\rmKL_i(\mu_i , \mu_j ) - n_j (\bmu) \rmKL_j (\mu_j, \mu_i)\right)}{8} \Ind\bracket{|\mu_i|,|\mu_j|\leq \frac{1}{\sqrt{\tilde\Delta}}} \diff H(\bmu) \label{appeqn: lower bound error prob}
\end{align}
and the following statement is a stronger statement than Theorem \ref{lem:k_asym}.
\begin{center}
    \text{If} $ \int \paren{\sum_{i\in[k]} n_i (\bmu)} \diff  \bH(\bmu) \leq N_V,$
    \text{then} (RHS of Eq. \eqref{appeqn: lower bound error prob})$\geq\delta$.
\end{center}
{
RHS of Eq. \eqref{appeqn: lower bound error prob} is represented in terms of $n:=(n_1, \cdots, n_k): \mathbb{R}^k \to [0,\infty)^k$ (the expected number of arm pulls) and does hold for any algorithm given $n$.
}
To prove the above statement, since \brown{the} `set of all expected number of arm pulls' is a subset of $\{\tilde{n}: \mathbb{R}^k \to [0,\infty)\}$, it suffices to show that the optimal value $V$ of the following objective

\begin{align} \label{K-arm opt problem 0}
V^{\mathrm{min}} &:= \inf_{\tilde{n}:\mathbb{R}^k \to [0,\infty)^k}\ V(\tilde{n}) 
\\
\text{s.t.\ \ } & 
\int_{\mathbb{R}^k} \paren{\sum_{s=1}^k \tilde{n}_s(\bmu)} {\diff \bH(\bmu)} \leq N_V  \nonumber
\\
\text{where \ \ } & V(\tilde{n}) := e^{-4} \sum_{i\in [k]}\sum_{j\neq i} \int_{\Theta_{ij}}
\frac{\exp\left(-\tilde{n}_i(\bmu)\rmKL_i(\mu_i , \mu_j ) - \tilde{n}_j (\bmu) \rmKL_j (\mu_j, \mu_i)\right)}{8} \Ind\bracket{|\mu_i|,|\mu_j|\leq \frac{1}{\sqrt{\tilde\Delta}}}\diff  H(\bmu) \nonumber
\end{align}
is greater than $\delta$.

Let $\tilde{\Theta}_{ij}:=\{\bmu\in \Theta_{ij}: |\mu_i -\mu_j|\leq \tilde\Delta , |\mu_i|\leq \frac{1}{\sqrt{\tilde{\Delta}}}, |\mu_j| \leq \frac{1}{\sqrt{\tilde{\Delta}}}\}$. Then, it holds that $V^{\min}\geq V_1^{\min}$ where

\begin{align}\label{K-arm opt problem 2}
V_1^{\min} &:= \inf_{\tilde{n} {: \mathbb{R}^{k} \to [0,\infty)^k}}\ \ V_1 (\tilde{n})
\\
\text{s.t.\ \ } & 
\sum_{i\in[k]}\sum_{j \neq i} \int_{\tilde{\Theta}_{ij}} \paren{\sum_{s=1}^k \tilde{n}_s(\bmu)} \diff \bH(\bmu) \leq N_V \nonumber
\\
\text{where \ \ }& V_1 (\tilde{n}):=e^{-4} \sum_{i\in [k]}\sum_{j\neq i} \int_{\tilde\Theta_{ij}}
\frac{\exp\left(-\tilde{n}_i(\bmu)\rmKL_i(\mu_i , \mu_j ) - \tilde{n}_j (\bmu) \rmKL_j (\mu_j, \mu_i)\right)}{8} \diff H(\bmu). \nonumber
\end{align}

 To see this, suppose $\hat{n}$ is an optimal solution to \eqref{K-arm opt problem 0}. Then, by the constraint of optimization problem \eqref{K-arm opt problem 0}, $\int_{{\mathbb{R}^k}} \paren{\sum_{s\in[k]} \hat{n}_s(\bmu)} \diff \bH(\bmu) \leq N_V$ and since $\hat{n}$ is a collection of positive functions, $\sum_{i,j\in[k]: i>j} \int_{\tilde{\Theta}_{ij} \cup \tilde{\Theta}_{ji}} \paren{\sum_{s=1}^k \tilde{n}_s(\bmu)} \diff \bH(\bmu) \leq N_V \nonumber
$, which means $\hat{n}$ satisfies the constraint of \eqref{K-arm opt problem 2}. By the minimality, $V_1^{\min} \leq V_1 (\hat{n})$, and since $\exp$ is a positive function, we have $V_1 (\hat{n}) \leq V(\hat{n}) = V^{\min}$. 
 

Moreover, $V_1^{\min} \ge V_2^{\min}$ holds for 
\begin{align}\label{K-arm opt problem 3}
V_2^{\min} &:= \inf_{\tilde{n} {: \mathbb{R}^k \to [0,\infty)^k}}\ \  V_2 (\tilde{n}) \\
\text{s.t.\ \ } & 
\sum_{i\in[k]}\sum_{j \neq i} \int_{\tilde{\Theta}_{ij}} \paren{\sum_{s=1}^k \tilde{n}_s(\bmu)} \diff \bH(\bmu) \leq N_V \nonumber
\\
\text{where \ \ }& V_2 (\tilde{n}):=
e^{-4} 
\sum_{i\in[k]}\sum_{j \neq i} \int_{\tilde{\Theta}_{ij}}\frac{\exp\left(-\paren{\tilde{n}_i (\bmu) + \tilde{n}_j (\bmu)}\frac{\tilde\Delta^2}{2\min(\sigma_s^2)_{s\in[k]}} \right)}{8} \diff  \bH(\bmu)  \nonumber
\end{align}

by using the fact that $\rmKL_i(\mu_i, \mu_j) = \frac{\tilde\Delta^2}{2\sigma_i^2} \leq \frac{\tilde\Delta^2}{2\min(\sigma_s^2)_{s\in[k]}}$.
\begin{claim} \label{claim: optimal solution}We abuse our notation {slightly} so that $\bH(E) = \int_{E} \diff \bH(\bmu)$ for any Lebesgue measurable set $E$, and let $\tilde{\Theta}= \cup_{i,j \in [k]: i \neq j} \tilde{\Theta}_{ij}$ (note that all $\tilde{\Theta}_{ij}$ are mutually disjoint except for the measure zero sets). Then, the following ${n}^{opt}$ is an optimal solutions to \eqref{K-arm opt problem 3} 
$${n}_s^{opt} (\bmu) :=\frac{N_V}{2 \bH (\tilde{\Theta}) } \Ind \bracket{\bmu \in \paren{\cup_{j\neq s} \tilde\Theta_{sj}} \cup \paren{\cup_{i\neq s} \tilde\Theta_{is}}}.$$
\end{claim}
\begin{proof}
    Choose an arbitrary $\tilde{n} : \mathbb{R}^k \to [0, \infty)^k$ which satisfies the constraint of optimization problem \eqref{K-arm opt problem 3}. Let $\tilde{N}(\bmu) := \bracket{\sum_{s\in[k]} \tilde{n}_s (\bmu)}$. 
    Now, since the function $\rho: x \mapsto \frac{1}{8R}\exp (-x \cdot \frac{\tilde\Delta^2}{2 \min (\sigma_s^2)_{s\in [k]}})$ is a convex and decreasing function, by Jensen's inequality we can say that
    \begin{align*}
        V_2(\tilde{n}) &= \sum_{i\in[k]}\sum_{j\neq i} \int_{\tilde\Theta_{ij}} \rho\paren{\tilde{n}_i (\bmu)+ \tilde{n}_j (\bmu)} \diff \bH (\bmu) \\
        &\geq \sum_{i\in[k]}\sum_{j\neq i} \bH(\tilde{\Theta}_{ij}) \rho \paren{\frac{1}{\bH(\tilde{\Theta}_{ij})} \int_{\tilde{\Theta}_{ij}}\paren{\tilde{n}_i (\bmu)+ \tilde{n}_j (\bmu)} \diff \bH (\bmu) } \tag{Jensen's inequality for each integral on $\tilde{\Theta}_{ij}$}\\
        &= \bH(\tilde{\Theta}) \cdot \sum_{i\in[k]}\sum_{j\neq i} \frac{\bH(\tilde{\Theta}_{ij})}{\bH(\tilde{\Theta})} \rho \paren{\frac{1}{\bH(\tilde{\Theta}_{ij})} \int_{\tilde{\Theta}_{ij}}\paren{\tilde{n}_i (\bmu)+ \tilde{n}_j (\bmu)} \diff \bH (\bmu) } \\
        &\geq \bH(\tilde{\Theta}) \cdot \rho \paren{\sum_{i\in[k]}\sum_{j\neq i} \frac{\bH(\tilde{\Theta}_{ij})}{\bH(\tilde{\Theta})}\frac{1}{\bH(\tilde{\Theta}_{ij})} \int_{\tilde{\Theta}_{ij}}\paren{\tilde{n}_i (\bmu)+ \tilde{n}_j (\bmu)} \diff \bH (\bmu) } \tag{Jensen's inequality}\\
        & \geq \bH(\tilde{\Theta}) \cdot \rho \paren{\sum_{i\in[k]}\sum_{j\neq i} \frac{1}{\bH(\tilde{\Theta})} \int_{\tilde{\Theta}_{ij}}\paren{\sum_{s\in[k]} \tilde{n}_s} \diff \bH (\bmu) } \tag{$\tilde{n}_i + \tilde{n}_j \leq \sum_{s\in[k]} \tilde{n}_s$ and $\rho$ is a decreasing function.}\\
        & \geq \bH(\tilde{\Theta}) \cdot \rho \paren{\frac{N_V}{\bH(\tilde{\Theta})}}. \tag{Constraint of \eqref{K-arm opt problem 3}}
    \end{align*}
    
    Since $\tilde{n}$ is chosen arbitrarily, we can say $V_2^{\min} \geq \bH(\tilde\Theta)\rho\paren{\frac{N_V}{\bH(\tilde\Theta)}}$. One can check the above $n^{opt}$ satisfies the constraint in the optimization problem \eqref{K-arm opt problem 3} and also satisfies $V_2 (n^{opt})=\bH(\tilde{\Theta})\rho(\frac{N_V}{\bH(\tilde{\Theta})})$. Therefore, $n^{opt}$ is an optimal solution of optimization problem \eqref{K-arm opt problem 3}.
\end{proof}
Using Lemma \ref{lem:volume_rewrite}, we can get \begin{align}\label{appeqn: volume of tilde theta}
    H(\tilde{\Theta})=\sum_{i\neq j} \int_{\Theta_{ij}} \Ind[|\mu_i- \mu_j| \le \tilde\Delta] \Ind[|\mu_i|, |\mu_j|\leq \frac{1}{\sqrt{\tilde\Delta}}]  \diff \bH(\bmu) = \LijHdelp{\tilde\Delta}.
\end{align} 
Now applying Claim \ref{claim: optimal solution} and Eq. \eqref{appeqn: volume of tilde theta} on Optimization \eqref{K-arm opt problem 3} implies the following result:
\begin{align*}
V_2^{\min} &=V_2 (n^{opt}) \tag{Claim \ref{claim: optimal solution}}\\
&= \frac{\exp\left(-\frac{N_V}{2 \LHdelp{\tilde\Delta}} \frac{\tilde\Delta^2}{2\min(\sigma_s^2)_{s\in[k]}} \right)}{8e^4} 
\sum_{i\neq j} \int_{\Theta_{ij}} \Ind[{|\mu_i - \mu_j| \le \tilde\Delta}]\Ind[|\mu_i|, |\mu_j|\leq \frac{1}{\sqrt{\tilde\Delta}}] \diff H(\bmu) 
 \tag{Eq. \eqref{appeqn: volume of tilde theta}}\\
&\ge \frac{\exp\left(-\frac{N_V\tilde\Delta}{2\min(\sigma_s^2)_{s\in[k]}\LHdelp{\tilde\Delta}}\right)}{8e^4} \times \LHdelp{\tilde\Delta} \tilde\Delta.  \tag{Eq. \eqref{appeqn: volume of tilde theta}}
\end{align*}

If $V\leq \delta$ is true, then $V_2 \leq \delta$, which implies 

$$\frac{\exp\left(-\frac{N_V\tilde\Delta}{2\min(\sigma_s^2)_{s\in[k]}\LHdelp{\tilde\Delta}}\right)}{8e^4} \times \LHdelp{\tilde\Delta} \leq \delta \Longleftrightarrow N_V \geq \frac{2 \min(\sigma_s^2)_{s\in[k]} \LHdelp{\tilde\Delta}}{\tilde{\Delta}}\ln \frac{\LHdelp{\tilde\Delta}}{8e^4\delta}.$$
{
To make this lower bound greater than 0, $\ln \frac{\LHdelp{\tilde\Delta}}{8e^4 \delta} > 1$. From Lemma \ref{lem:volume_rewrite}, we know that for small enough $\tilde{\Delta}$\footnote{$\tilde{\Delta}<\min\paren{D_0(\bH), \min_{i \in [k]}\frac{1}{4m_i^2}, \frac{L(\bH)}{4(k-1)\sum_{i \in [k]}\frac{1}{\xi_i}}}$. Check Section \ref{appsec:delta0} for the condition}, $\LHdelp{\tilde{\Delta}} \in [\frac{1}{2}L(\bH), 2L(\bH)]$.
Setting \begin{equation}\label{appeqn:tildedelta}
    \tilde\Delta:=\frac{32e^4}{L(\bH)} \delta,
\end{equation} we have 
\begin{equation}\label{appeqn:delta1 usage}
    N_V \geq \min(\sigma_s^2)_{s\in[k]} \frac{L(\bH)^2}{16e^4}\ln 2
\end{equation}
which is the inequality we desired. }
\end{proof}

\subsection{Proof of Lemmas}

\subsubsection{Proof of lemma \ref{lem: k-error-2}}\label{subsec_kerror}
\begin{proof} 
It is equivalent to prove $ {p+q}\geq \frac{1}{4} \exp(-d(p,1-q'))$ for any $p, q, q' \in (0, 1)$ such that $q'\leq q$.
First, if $p \ge 1/4$ or $q \ge 1/4$ it trivially holds, and thus we assume $p < 1/4$ and $q < 1/4$.
We have 
\begin{align}
    d(p, 1-q') & d(p, 1-q) \tag{$p,q \leq \frac{1}{4}$}\\
    &\leq d(p+q, 1-(p+q)) \tag{$p+q < \frac{1}{2}$}\\
    &\leq \log \frac{1}{2.4 (p+q)} \tag{Eq.(3) of \cite{kaufman16a}}
\end{align}
and transforming this yields 
\[
{p+q}
\ge \frac{1}{2.4} e^{-(d(p, 1-q'))}.
\]
{This completes the proof.}
\end{proof}

\subsubsection{Proof of the Lemma \ref{lem:ratiolemma}}\label{subsec_ratiolemma}
\begin{proof}
{We have}
\begin{align*}
    \frac{h_i (a) h_j(b)}{h_i(b)h_j(a)} &= \exp\paren{-\frac{(a-m_i)^2}{2\sigma_i^2}+\frac{(b-m_i)^2}{2\sigma_i^2}-\frac{(b-m_j)^2}{2\sigma_j^2}+\frac{(a-m_j)^2}{2\sigma_j^2}}\\
    &=\exp\paren{-\frac{(a-b)(a+b-2m_i)}{2\sigma_i^2}-\frac{(b-a)(a+b-2m_j)}{2\sigma_j^2}} \\
    &= \exp\paren{2(a-b)\bracket{\frac{m_i}{\sigma_i^2} -\frac{m_j}{\sigma_j^2}}- (a-b)(a+b)\bracket{\frac{1}{2\sigma_i^2}+\frac{1}{2\sigma_j^2}}}\\
    &\geq \exp\paren{-2|a-b|\left|\frac{m_i}{\sigma_i^2} -\frac{m_j}{\sigma_j^2}\right|- |(a-b)(a+b)|\bracket{\frac{1}{2\sigma_i^2}+\frac{1}{2\sigma_j^2}}}\\
    &\geq \exp\paren{-2\delta\left|\frac{m_i}{\sigma_i^2} -\frac{m_j}{\sigma_j^2}\right|- 2\sqrt{\delta}\bracket{\frac{1}{2\sigma_i^2}+\frac{1}{2\sigma_j^2}}} {\brown{=:} \frac{1}{R_{ij}'(\bH, \delta)}}.
\end{align*}
Define $R'(\bH, \delta) = \min_{i\neq j} R_{ij}'(\bH, \delta)$, and this $R'(\bH, \delta)$ satisfies the condition of the Ratio lemma. 
{For $\delta$ such that}
$$\delta < D_1 (\bH) := \min_{i \neq j} \bracket{\left|\frac{m_i}{\sigma_i^2} -\frac{m_j}{\sigma_j^2}\right|^{-1}, \bracket{\frac{1}{2\sigma_i^2}+\frac{1}{2\sigma_j^2}}^{-2}},$$
we have $R'(\bH, \delta) \leq e^{4}$. 

\end{proof}

\section{Proof of Theorem \ref{thm: elim upperbound}}\label{appsec: upper bound proof}

For the notational convenience, Define $\Delta_s (\bmu) := \mu_{i^*(\bmu)} - \mu_s$ for $s\in [k]$, $\Delta (\bmu) := \min_{s \ne i^*(\bmu)} \Delta_s = \Delta_{j^*(\bmu)}$. Let $\EA(t)$ be the subset of arms that have not been eliminated at time $t$. 

\newcommand{\uhdel}{u_{\bH}(\Deltatarget)}
\newcommand{\lhdel}{\ell_{\bH}(\Deltatarget)}
Since $\delta< \frac{4L^2(\bH)}{{\sum_{i\in[k]}\frac{k-1}{\xi_i}}} $,  $\Deltatarget =\frac{\delta}{4L(\bH)}\le\frac{L(\bH)}{\sum_{i\in[k]}\frac{k-1}{\xi_i}}$ and we have
\begin{align}
\Prob_{\bmu \sim \bH} (\Delta_0 \geq \Delta(\bmu)) &= \sum_{i \neq j} \int_{\Theta_{ij}} \Ind[|\mu_i - \mu_j| \leq \Deltatarget] \diff \bH(\bmu) \nonumber\\
=&\LHdel{\Deltatarget} \Deltatarget \leq 2L(\bH) \cdot \Deltatarget
\tag{Lemma \ref{lem:volume_rewrite}}\\
&\leq \frac{\delta}{2}. \label{appeqn: definition of Deltatarget}
\end{align}

Next, we consider the upper bound estimator such that $\hat{\Delta}^{\mathrm{safe}}(t) \ge \Delta$ holds with high probability. Namely, 
\begin{align*}%
\rmUCB(i,t), \rmLCB(i,t) 
&= \hatmu_i(t) \pm \Conf(i, t)\\
\Conf(i, t) &= \sqrt{2\sigma_i^2 \frac{\log(6(N_i(t))^2/((\frac{\deltatwo}{2k })\pi^2))}{N_i(t)}}\\
\hat{\Delta}^{\mathrm{safe}}(t) &= \max_i \rmUCB(i,t) -\max_j \rmLCB(j,t).
\end{align*}


From the definition above, we can calculate how many arm pulls the learner needs to narrow down the confidence width. 
\begin{lem}\label{lem_confwidth}
    Define $B := 320{\max_{i \in [k]}\sigma_i^2}$ and $\DeltaThr:= \min\paren{\log \frac{4\sqrt{k }}{\delta \pi}, \frac{1}{B}}$. Let $R_0 (\Delta) := B \frac{\log {\min(\Delta, \DeltaThr)^{-1}}}{\min(\Delta, \DeltaThr)^2}$. Then, for any $\Delta\in (0, \infty)$ and for a timestep $t$ which satisfies $N_i (t) \geq R_0(\Delta)$, $\Conf (i, t) \leq {\Delta/4}$.
\end{lem}

\begin{proof}[Proof of Lemma \ref{lem_confwidth}]
Only for this part of the proof, let $\Delta'=\min(\Delta, \DeltaThr)$ for notational convenience. Then, 
\begin{align*}
\Conf(i,t) &= \sqrt{2\sigma_i^2 \frac{\log(6(N_i(t))^2/((\frac{\deltatwo}{2k })\pi^2))}{N_i(t)}} \\
&\leq \sqrt{2\sigma_i^2 \frac{\log(6R_0 (\Delta)^2/((\frac{\deltatwo}{2k })\pi^2))}{R_0(\Delta)}} \tag{assumption on $t$}\\
&= {\Delta'} \sqrt{2\sigma_i^2}\sqrt{ \frac{\log(6(B \times \frac{ \log({\Delta'}^{-1})}{{\Delta'}^2})^2/((\frac{\deltatwo}{2k })\pi^2))}{B \times { \log({\Delta'}^{-1})}}}\\
&\leq {\Delta'} \sqrt{2\sigma_i^2}\sqrt{ \frac{\log(6(\frac{B}{{\Delta'}^3})^2/((\frac{\deltatwo}{2k})\pi^2))}{B { \log({\Delta'}^{-1})}}} \tag{$\log {\Delta'}^{-1} \leq {\Delta'}^{-1}$}\\
&={\Delta'} \sqrt{2\sigma_i^2}\sqrt{ \frac{6 \log ({\Delta'}^{-1}) +2 \log B +\log \frac{12k}{\deltatwo \pi^2}}{B { \log({\Delta'}^{-1})}}}
\\
&\leq \Delta' \sqrt{2\sigma_i^2} \sqrt{\frac{6}{B} + \frac{2}{B} + \frac{2}{B}} \tag{Definition of $\DeltaThr$, $\DeltaThr \geq \Delta'$}
\\
&= \Delta' \sqrt{\frac{20\sigma_i^2}{B}}\leq \frac{1}{4}\Delta'\leq \frac{1}{4}\Delta.
\end{align*}
\end{proof}

From this lemma, one could induce the following corollary which states that Algorithm 5 always terminates before a certain timestep:

\begin{cor}\label{lem_panic}
Let $\tau$ (and $\gamma$) be the stopping time (and the last iteration of the while loop in Algorithm \ref{alg:elim}, respectively) where Algorithm \ref{alg:elim} meets the stopping condition. Then, $\gamma$ is always bounded by $R_0(\Deltatarget)$ and $\tau$ is uniformly bounded by $ T_0
:= k \cdot R_0(\Deltatarget)$.
\end{cor}

\begin{proof}[Proof of Corollary \ref{lem_panic}]
Let us assume that $\tau > T_0+1$. Then, by Lemma \ref{lem_confwidth}, each $i\in \mathcal{A}_{T_0}$ satisfies $\Conf(i, T_0) \le \Deltatarget/4$. Let $i^{ucb}(t)=\arg\max_{i \in \mathcal{A}(t)} \rmUCB(i,t)$ and $i^{lcb}(t)=\arg\max_{i \in \mathcal{A}(t)} \rmLCB(i,t)$. From definition, 
\begin{align*}
\hat{\Delta}^{\mathrm{safe}} (T_0)&=\max_{i \in \mathcal{A}(T_0)} \rmUCB(i,T_0)-\max_{i \in \mathcal{A}(T_0)} \rmLCB(i,T_0) \\
    &= \rmUCB(i^{ucb}(T_0),T_0) - \rmLCB(i^{lcb}(T_0),T_0) \\
    &= 2\Conf(i^{ucb}(T_0),T_0) + 2\Conf(i^{lcb}(T_0),T_0) + \rmLCB(i^{ucb}(T_0),T_0) - \rmUCB(i^{lcb}(T_0),T_0)\\
    &\leq \Delta_0 +\rmLCB(i^{ucb}(T_0),T_0) - \rmUCB(i^{lcb}(T_0),T_0) \\
    &\leq \Delta_0 \tag{Since both arms survived from the elimination phase.}
\end{align*} 
which implies $\hat{\Delta}^{\mathrm{safe}} \le \Deltatarget$, {which contradicts}
$\tau \geq T_0$ since the Algorithm should be {terminated} by Line 18 at timestep $T_0$. Therefore, $\tau \leq T_0$.
\end{proof} 

{Lemma \ref{lem_panic} implies} Algorithm \ref{alg:elim} always stops before $T_0$ samples. {Morever, the following} lemma states that with high probability, true mean $\bmu$ is in between $\rmUCB$ and $\rmLCB$ for all time steps.

\begin{lem}{\rm (Uniform confidence bound)}\label{lem_cb} The following holds for all $i \in [k]$:
\begin{align}
\Prob_{\bmu} \left[ \bigcap_{t=1}^{\infty} \left\{\rmLCB(i,t) \le \mu_i \right\}\right] 
&\ge 1 - \frac{\deltatwo}{2k}, \label{eq. lem_cb_lb}\\
\Prob_{\bmu} \left[ \bigcap_{t=1}^{\infty} \left\{\mu_i \le \rmUCB(i,t) \right\}\right] 
&\ge 1 - \frac{\deltatwo}{2k}. \label{eq. lem_cb_ub}
\end{align}
\end{lem}

\begin{proof}[Proof of Lemma \ref{lem_cb}]

{The following derives} the upper bound part, Eq. \eqref{eq. lem_cb_ub}. The lower bound {is derived by following the same steps.}

Since each arm is independent of each other, Eq. \eqref{eq. lem_cb_ub} boils down to prove

$$\Prob_{\bmu}\bracket{\bigcap_{s=1}^{\infty} \{ {\mu}_i  \leq \rmUCB(i, t_i(s))\}} \geq 1- \frac{\deltatwo}{2k},$$

where $t_i (s) = \min \{t\in \mathbb{N}: N_i (t) \geq s\}$ and $t_i (s) = \infty$ if $\{t\in \mathbb{N}: N_i (t) \geq s \} = \emptyset$. For each event $\{{\mu}_i  \leq UCB_i (t_i(s))\}$, since each arm pull is independent of each other, by Hoeffding's inequality we have

$$ \Prob_{\bmu} \paren{\frac{1}{s}\sum_{j=1}^s (X^{i}_j - \mu_i) \leq -\epsilon} \leq \exp\paren{-\frac{s\epsilon^2}{2\sigma_0^2}} $$
for any $\epsilon>0$. If we set $\epsilon = \Conf(i, t_i(s))$, we can transform the above inequality to 
\begin{align*}
    \Prob_{\bmu} \paren{\rmUCB(i,t_i(s))\leq \mu_i} \leq \frac{\deltatwo}{2ks^2} \cdot \frac{6}{\pi^2}.
\end{align*}

By the union bound,

\begin{align*}
    \Prob_{\bmu} \bracket{\bigcap_{s=1}^{\infty} \{ {\mu}_i \leq \rmUCB(i, t_i(s))\}}& \geq 1 - \sum_{s=1}^{\infty} \Prob_{\bmu} \bracket{ {\mu}_i \geq \rmUCB(i, t_i(s))} \\
    &\geq 1- \sum_{s=1}^{\infty} \frac{\deltatwo}{2ks^2}\cdot \frac{6}{\pi^2}\\
    &\geq 1- \sum_{s=1}^{\infty} \frac{\deltatwo}{2ks^2}\cdot \frac{6}{\pi^2} = 1-\frac{\deltatwo}{2k},
\end{align*}
and the proof is completed. 
\end{proof}

Let us define a good event based on Lemma \ref{lem_cb}. 
\begin{equation}\label{ineq_goodevent}
\mathcal{X} (\bmu):= \bigcap_{i\in[k]} \left[\paren{\bigcap_{t=1}^\infty\left\{\rmLCB(i,t) \le \mu_i \right\}} {\bigcap} \paren{\bigcap_{t=1}^{\infty}\left\{\rmUCB(i,t) \ge \mu_i \right\}}\right].
\end{equation}

{We now} prove that, under $\EH_{\bmu}$ and under this good event $\EX(\bmu)$ 
\begin{itemize}
    \item The best arm $i^*(\bmu)$ is always in the active arm set $\EA(t)$ for all $t$ (Lemma \ref{lem_wrongdrop})
    \item Each count of the suboptimal arm pull, $N_i (T_0)$, is bounded by roughly $O(\frac{\log \Delta_i}{\Delta_i^2})$ (Lemma \ref{lem_exp_subopt}).
\end{itemize}

\begin{lem}\label{lem_wrongdrop}
Let
\begin{equation}
\EX' (\bmu)= 
\bigcap_t \left\{
i^*(\bmu) \in \EA(t)
\right\}.
\nonumber
\end{equation}
Then, under $\EH_{\bmu}$, $\EX (\bmu)\subset \EX' (\bmu)$ and naturally
\begin{equation}
\Prob_{\bmu}
\left[ 
(\EX'(\bmu))^c
\right]
\le \deltatwo. 
\nonumber
\end{equation}
\end{lem}

\begin{proof}[Proof of Lemma \ref{lem_wrongdrop}]

Suppose that event $\EX (\bmu)$ occurs under $\EH_{\bmu}$. Then for all $i \in [k]$ and for all $t$, $\hat{\mu}_{i} - \Conf (i,t) \leq \mu_i$ and $\hat{\mu}_{i} + \Conf (i,t) \geq \mu_i$. Now, for any $i \neq i^*$,

\begin{align}
    \rmLCB(i,t) - \rmUCB(i^*, t) &= \hat{\mu}_{i} - \Conf(i,t) - (\hat{\mu}_{i^*} + \Conf(i^*,t))\nonumber\\
    &\leq \mu_i + \Conf(i,t) -\Conf(i,t) - (\mu_{i^*} - \Conf(i^*,t)+\Conf(i^*,t))\tag{Event $\EX$ occurs}\\
    &= \mu_i - \mu_{i^*} < 0 \nonumber
\end{align}

which means when event $\EX$ occurs, the optimal arm will never be dropped, and thus $\EX \subset \EX'$. By Lemma \ref{lem_cb},
$$\Prob_{\bmu}(\EX'(\bmu)) \geq \Prob_{\bmu}(\EX(\bmu)) \geq 1-\deltatwo. $$
\end{proof}


\begin{lem}\label{lem_exp_subopt}
For any $i \ne i^*$, under $\EH_{\bmu}$ we have
\begin{equation}
\left\{
N_i(T_0)> R_0 \paren{\max(\Delta_i,\Delta_0)}
\right\}
\subset \EX(\bmu)^c,
\end{equation}
and therefore, $\Ex_{\bmu}[N_i(T_0) \Ind[\EX(\bmu)]] \leq R_0 \paren{\max(\Delta_i,\Delta_0)}$.
\end{lem}
\begin{proof}[Proof of Lemma \ref{lem_exp_subopt}]
Only for this part of the proof, let $T_i := R_0 \paren{\max(\Delta_i,\Delta_0)}$ for brevity.

When $\Delta_i < \Deltatarget$, $\max(\Delta_i, \Deltatarget)=\Deltatarget$, {which, combined with Corollary \ref{lem_panic}, implies that $\left\{N_i(T_0)\geq T_i \right\}$ always holds.}

For the case of $\Delta_i > \Deltatarget$, suppose that the learner is under the events $\EX(\bmu)$ and $\{i \in \mathcal{A}(T_i)\}$.  Note that 
\[
\Conf(a, T_i) \le \frac{\Delta_i}{4}, \quad \forall a \in \mathcal{A}(T_i).
\]

Then, 
\begin{align*}
    \max_{a \in \mathcal{A}(T_i)} \rmLCB_a (T_i) - \rmUCB_i (T_i) &\geq \rmLCB_{i^*} (T_i) - \rmUCB_i (T_i) \\
    &= \hat{\mu}_{i^*} - \Conf_{i^*} (T_i) - (\hat{\mu}_i +\Conf_i (T_i)) \\
    &\geq \mu_{i^*} - 2\Conf_{i^*}(T_i) - ({\mu}_i +2\Conf_i (T_i)) \tag{$\EX$ occurs}\\
    &\geq \mu_{i^*}-\mu_i - \Delta_i =0.
\end{align*}
Therefore, when $\EX$ occurs, $\mu_i$ should be eliminated after timestep $T_i$ so $N_i (T_0) \leq T_i$. 
\end{proof} 

Lemmas \ref{lem_wrongdrop} and \ref{lem_exp_subopt} guarantee the Bayesian $\delta$-correctness of our Algorithm \ref{alg:elim}. 

\begin{thm}{\rm ($\delta$-correctness)}\label{thm_correct}
The Bayesian PoE of Algorithm \ref{alg:elim} is at most $\delta$.
\begin{equation}
\int_{{\mathbb{R}^k}} \mathrm{PoE}(\bmu) \diff \bH(\bmu) \le \delta.
\end{equation}
\end{thm}

\begin{proof}[Proof of Theorem \ref{thm_correct}]
Throughout the proof, we now have the following results.
\begin{itemize}
\item The probability that $\mu_{i^*(\bmu)}-\mu_{j^*(\bmu)}\leq \Deltatarget$ is at most $\frac{\delta}{2}$ by Eq. \eqref{appeqn: definition of Deltatarget}. 
\item The event $\EX$ failed to hold with probability at most $\deltatwo$ by Lemma \ref{lem_cb}. 
\item When $\mu_{i^*(\bmu)}-\mu_{j^*(\bmu)}>\Deltatarget$ and event $\EX$ occurs, by Lemmas \ref{lem_wrongdrop} and \ref{lem_exp_subopt}, all suboptimal arms will be eliminated before $T_{j^*}$ and only the optimal arm will remain eventually. This means $\EE(\bmu)\subset  \EX (\bmu) \cup \{\mu_{i^*(\bmu)}-\mu_{j^*(\bmu)}\leq \Deltatarget\}$.
\end{itemize}

Therefore, 

\begin{align*}
    \mathrm{PoE} (\pi;\bH) &= \Ex_{\bmu \sim \bH} \bracket{\Prob\paren{J\neq i^*{(\bmu)} | \EH_{\bmu}}}= \Ex_{\bmu\sim\bH}\bracket{\Ex[\Ind(\EE(\bmu))|\EH_{\bmu}]}\\
    &\leq \Ex_{\bmu\sim\bH}\bracket{\Ex[\Ind(\{\mu_{i^*(\bmu)}-\mu_{j^*(\bmu)}\leq \Deltatarget\})+ \Ind(\EX(\bmu))|\EH_{\bmu}]}\\
    &\leq \Ex_{\bmu\sim\bH}\bracket{\Ind(\{\mu_{i^*(\bmu)}-\mu_{j^*(\bmu)}\leq \Deltatarget\})+ \Ex[\Ind(\EX(\bmu))|\EH_{\bmu}]}\\
    &\leq \frac{\delta}{2}+\deltatwo<\delta.
\end{align*}
\end{proof}

\newcommand{\lhdelz}{L({\bH},\Deltatarget)}

Finally, Lemma \ref{lem: overall expectation} shows the upper bound of the expected stopping time of our algorithm.

\begin{lem} \label{lem: overall expectation}
    {We have} $\Ex[\tau] \leq 
    {B_0 \frac{L(\bH)^2}{\delta} \log\left( \frac{L(\bH)}{\delta}\right)}
        + O(\log \delta^{-1}).$
\end{lem}

\begin{proof} 
{We have }
\begin{align*}
    \Ex[\tau] &= \sum_{s=1}^k \Ex[N_s (T_0)] = {\sum_{s=1}^k}\Ex[\Ex_{\bmu} [N_s(T_0)]] \\
    &= \sum_{s=1}^k\Ex\bracket{\Ex_{\bmu}[N_s (T_0) \Ind[\EX(\bmu)]]} +\sum_{s=1}^k \Ex\bracket{\Ex_{\bmu}[N_s (T_0) \Ind[\EX^c(\bmu)]]}\\
    &\leq \sum_{s=1}^k \Ex\bracket{\Ex_{\bmu}[N_s (T_0) \Ind[\EX(\bmu)]]} + T_0 \cdot \deltatwo \tag{Corollary \ref{lem_panic} and Lemma \ref{lem_cb}}
\end{align*}

{and since $\Deltatarget< \DeltaThr$ by assumption, \begin{equation}\label{appeqn: expected stopping time minor part}
    T_0 \cdot \deltatwo = k\cdot R_0(\Deltatarget) \cdot \deltatwo \leq kB \cdot \frac{\log \Deltatarget^{-1}}{\Deltatarget^2} \cdot \deltatwo.
\end{equation} } Therefore, it remains to compute the scale of the first term, $\sum_{s=1}^k \Ex\bracket{\Ex_{\bmu}[N_s(T_0) \Ind[\EX(\bmu)]}$.\\

{The following evaluates} $\Ex\bracket{\Ex_{\bmu}[N_i (T_0) \Ind[\mathcal{X}(\bmu)]]}$ for each $i$. For notational convenience, let $\mathcal{T}_s (\bmu) = \Ex_{\bmu}\bracket{R_0(\max\paren{\Delta_s, \Delta_0)} \Ind[\EX(\bmu)]}$. 
Then, 

    \begin{align}
        \Ex\bracket{\Ex_{\bmu}[N_s (T_0) \Ind[\mathcal{X}(\bmu)]]} &\leq \Ex[\ET_s (\bmu)] \tag{Lemma \ref{lem_exp_subopt}}
        \\&= \sum_{i=1}^k \Ex[\ET_s(\bmu)\Ind[\bmu \in \Theta_{i}]] \nonumber
        \\&= \sum_{i=1}^k \int_{\mu \in \Theta_{i}} \ET_s (\bmu) \diff \bH (\bmu) \label{eqn:Sum of Case1 and Case2}
        \end{align}
(Recall that $\Theta_i=\{\bmu\in \mathbb{R}^k: \mu_i \geq \max_{j\neq i} \mu_j \}$). In the following, we calculate each $\int_{\mu \in \Theta_{i}} \ET_s (\bmu) \diff \bH (\bmu)$. 
{By assumption, $\Deltatarget < \DeltaThr$.} 


{For this part, we define a new notation that for each vector $v\in \mathbb{R}^k$ and $i,j \in [k]$, $v_{\backslash i}$ (and $v_{\backslash i,j}$) is the projection of $v$ to $\mathbb{R}^{k-1}$ ($\mathbb{R}^{k-2}$, respectively) by omitting $i$-th coordinate ($i,j$-th coordinate, respectively), $v_{\backslash i}:= \underbrace{(v_1, v_2, \cdots, v_{i-1},v_{i+1}, \cdots, v_k)}_{k-1\ \text{coordinates}}$. Similarly, for a $k$-dimensional distribution $\bH$ and $i,j \in [k]$, $\bH_{\backslash i}$ (and $\bH_{\backslash i, j}$) is the distribution which omits $i$-th ($i,j$-th, respectively) coordinate.}

\paragraph{Case 1: $s\neq i$} In this case, by Lemma  \ref{lem_exp_subopt},

\begin{align*}
    \int_{\mu \in \Theta_{i}} \ET_s (\bmu) \diff \bH (\bmu) =&     \int_{\mu \in \Theta_{i}} \ET_s (\bmu) (\Ind[\mu_i- \mu_s \geq \Deltatarget]+\Ind[\mu_i- \mu_s < \Deltatarget])\diff \bH (\bmu)\\
    =&{\int_{\mu_i \in \Real} \int_{\mu_s=-\infty}^{\mu_i - \DeltaThr} \int_{\mu_{\backslash i,s} \in (-\infty, \mu_i)^k} \ET_s (\bmu)  h_i (\mu_i)h_s (\mu_s)\diff \bH_{\backslash i,s} \diff \mu_s \diff \mu_i} \\
    &+\int_{\mu_i \in \Real} \int_{\mu_s=\mu_i - \DeltaThr}^{\mu_i - \Deltatarget} \int_{\mu_{\backslash i,s} \in (-\infty, \mu_i)^k} \ET_s (\bmu)  h_i (\mu_i)h_s (\mu_s)\diff \bH_{\backslash i,s} \diff \mu_s \diff \mu_i \\
    &+\int_{\mu_i \in \Real} \int_{\mu_s=\mu_i - \Deltatarget}^{\mu_i} \int_{\mu_{\backslash i,s} \in (-\infty, \mu_i)^k} \ET_s (\bmu)  h_i (\mu_i)h_s (\mu_s)\diff \bH_{\backslash i,s} \diff \mu_s \diff \mu_i \\
    \leq& B\cdot \frac{\log \DeltaThr^{-1}}{\DeltaThr^2}\\
    &+\int_{\mu_i \in \Real} \int_{\mu_s=\mu_i - \DeltaThr}^{\mu_i - \Deltatarget} \int_{\mu_{\backslash i,s} \in (-\infty, \mu_i)^k} B\frac{\log (\mu_i -\mu_s)^{-1}}{(\mu_i-\mu_s)^2}  h_i (\mu_i) h_s (\mu_s)\diff \bH_{\backslash i,s} \diff \mu_s \diff \mu_i \\
    &+\int_{\mu_i \in \Real} \int_{\mu_s=\mu_i - \Deltatarget}^{\mu_i} \int_{\mu_{\backslash i,s} \in (-\infty, \mu_i)^k} B\frac{\log \Deltatarget^{-1}}{\Deltatarget^2}  h_i (\mu_i)h_s (\mu_s)\diff \bH_{\backslash i,s} \diff \mu_s \diff \mu_i \tag{Lemma \ref{lem_exp_subopt}}\\
    = &\int_{\mu_i \in \Real} \int_{\mu_s=-\mu_i-\DeltaThr}^{\mu_i - \Deltatarget} B\frac{\log (\mu_i -\mu_s)^{-1}}{(\mu_i-\mu_s)^2}  h_i (\mu_i)h_s (\mu_s) \bracket{\prod_{k \in [k]\backslash \{i,s\}} H_k (\mu_i)} \diff \mu_s \diff \mu_i \\
    &+B\frac{\log \Deltatarget^{-1}}{\Deltatarget^2} \Prob(i^*(\mu) = i, \mu_i - \mu_s \leq \Deltatarget) +B\cdot \frac{\log \DeltaThr^{-1}}{\DeltaThr^2}\\
    \leq &B\log \Deltatarget^{-1} \underbrace{\int_{\mu_i \in \Real} h_i (\mu_i) \int_{\mu_s=-\infty}^{\mu_i - \Deltatarget}  \frac{1}{(\mu_i-\mu_s)^2} h_s (\mu_s) \bracket{\prod_{k \in [k]\backslash \{i,s\}} H_k (\mu_i)}\diff \mu_s \diff \mu_i}_{(A)} \\
    &+B\frac{\log \Deltatarget^{-1}}{\Deltatarget^2}\underbrace{\Prob(i^*(\mu) = i, \mu_i - \mu_s \leq \Deltatarget)}_{(P_{is})} + B \frac{\log \DeltaThr^{-1}}{\DeltaThr^{2}}.
\end{align*}

We first deal with the term $(A)$. The following splits $(A)$ into {the} sum of two integrals $(A1)$ and $(A2)$:
\begin{align*}
    (A) =&\int_{\mu_i \in \Real} h_i (\mu_i)\bracket{\prod_{k \in [k]\backslash \{i,s\}} H_k (\mu_i)} \int_{\mu_s=-\infty}^{\mu_i - \Deltatarget}  \frac{1}{(\mu_i-\mu_s)^2} h_s (\mu_s) \diff \mu_s \diff \mu_i \\
    =& \underbrace{\sum_{l=1}^{\ceil{\frac{1}{2\sqrt{\Deltatarget}}}-1} \int_{\mu_i \in \Real} h_i (\mu_i) \bracket{\prod_{k \in [k]\backslash \{i,s\}} H_k (\mu_i)}\int_{\mu_s=\mu_i - (l+1) \Deltatarget}^{\mu_i - l \Deltatarget} \frac{1}{(\mu_i-\mu_s)^2} h_s (\mu_s) \diff \mu_s \diff \mu_i}_{(A1)} \\
    &+ \underbrace{\int_{\mu_i \in \Real} h_i (\mu_i) \bracket{\prod_{k \in [k]\backslash \{i,s\}} H_k (\mu_i)}\int_{\mu_s=-\infty}^{\mu_i - \ceil{\frac{1}{2\sqrt{\Deltatarget}}}\Deltatarget}  \frac{1}{(\mu_i-\mu_s)^2} h_s (\mu_s) \diff \mu_s \diff \mu_i}_{(A2)}.
\end{align*}

For $(A1)$, we have
\begin{align}
    (A1) &\leq \sum_{l=1}^{\ceil{\frac{1}{2\sqrt{\Deltatarget}}}-1} \int_{\mu_i \in \Real} h_i (\mu_i) \bracket{\prod_{k \in [k]\backslash \{i,s\}} H_k (\mu_i)}\int_{\mu_s=\mu_i - (l+1) \Deltatarget}^{\mu_i - l \Deltatarget} \frac{1}{l^2 \Deltatarget^2} h_s (\mu_s) \diff \mu_s \diff \mu_i \nonumber\\
    &= \sum_{l=1}^{\ceil{\frac{1}{2\sqrt{\Deltatarget}}}-1} \frac{1}{l^2 \Deltatarget^2} \Prob(i^*(\bmu)=i, \mu_i - \mu_s \in [l\Deltatarget, (l+1)\Deltatarget]) \nonumber\\
    &\leq \sum_{l=1}^{\ceil{\frac{1}{2\sqrt{\Deltatarget}}}-1} \frac{1}{l^2 \Deltatarget^2} \cdot 2\Prob(i^*(\bmu)=i, \mu_i - \mu_s \leq \Deltatarget) \tag{Lemma \ref{lem:stratified volume}}\\    
    &\leq \sum_{l=1}^{\infty} \frac{1}{l^2 \Deltatarget^2} \cdot 2P_{is} \nonumber\\
    &= \frac{\pi^2}{3\Deltatarget^2} P_{is}. \label{eq:(A1) integral}
\end{align}

Now for $(A2)$, we evaluate the inner integral of $(A2)$:
\begin{align*}
    (Inner-A2) &:=  \frac{1}{\sqrt{2\pi}\sigma_s}\int_{-\infty}^{\mu_i-\ceil{\frac{1}{2\sqrt{\Deltatarget}}}\Deltatarget} \frac{1}{(\mu_i-\mu_s)^2}  \exp\paren{-\frac{(\mu_s-m_s)^2}{2\sigma_s^2}} \diff \mu_s \\
    &= \frac{1}{\sqrt{2\pi}\sigma_s}\biggl[\frac{1}{(\mu_i-\mu_s)} \exp\paren{-\frac{(\mu_s-m_s)^2}{2\sigma_s^2}}\biggr]_{-\infty}^{\mu_i-\ceil{\frac{1}{2\sqrt{\Deltatarget}}}\Deltatarget} \\
    &\ +\frac{1}{\sqrt{2\pi}\sigma_s}\int_{-\infty}^{\mu_i-\ceil{\frac{1}{2\sqrt{\Deltatarget}}}\Deltatarget} \frac{\mu_s-m_s}{\sigma_s^2(\mu_i-\mu_s)}  \exp\paren{-\frac{(\mu_s-m_s)^2}{2\sigma_s^2}} \diff \mu_s \tag{partial integration}\\
    &\leq \frac{1}{\sqrt{2\pi}\sigma_s} \frac{1}{\ceil{\frac{1}{2\sqrt{\Deltatarget}}}\Deltatarget} + \frac{1}{\sqrt{2\pi}\sigma_s^3}\int_{-\infty}^{\mu_i-\ceil{\frac{1}{2\sqrt{\Deltatarget}}}\Deltatarget} \paren{\frac{\mu_i-m_s}{\mu_i-\mu_s} -1} \exp\paren{-\frac{(\mu_s-m_s)^2}{2\sigma_s^2}} \diff \mu_s  \tag{$\mu_i > \mu_s$}\\
    &\leq \frac{1}{\sqrt{2\pi}\sigma_s}\frac{1}{\ceil{\frac{1}{2\sqrt{\Deltatarget}}}\Deltatarget} + \frac{1}{\sqrt{2\pi}\sigma_s^3}\int_{-\infty}^{\mu_i-\ceil{\frac{1}{2\sqrt{\Deltatarget}}}\Deltatarget} \frac{\mu_i-m_s}{\ceil{\frac{1}{2\sqrt{\Deltatarget}}}\Deltatarget} \exp\paren{-\frac{(\mu_s-m_s)^2}{2\sigma_s^2}} \diff \mu_s \\
    &\leq \frac{1}{\sqrt{2\pi}\sigma_s}\frac{1}{\ceil{\frac{1}{2\sqrt{\Deltatarget}}}\Deltatarget} + \max\bracket{\frac{1}{\sigma_s^2}\frac{\mu_i-m_s}{\ceil{\frac{1}{2\sqrt{\Deltatarget}}}\Deltatarget},0}.
\end{align*}

{By integrating above over variable $i$, we have}
\begin{align}
    (A2) 
    &\leq \int_{\Real} h_i (\mu_i) \bracket{\prod_{k \in [k]\backslash \{i,s\}} H_k (\mu_i)}\paren{\frac{1}{\sqrt{2\pi}\sigma_s}\frac{1}{\ceil{\frac{1}{2\sqrt{\Deltatarget}}}\Deltatarget} + \max\bracket{\frac{1}{\sigma_s^2}\frac{\mu_i-m_s}{\ceil{\frac{1}{2\sqrt{\Deltatarget}}}\Deltatarget},0}} \diff \mu_i \nonumber\\
    &\leq \int_{\Real} h_i (\mu_i) \paren{\frac{1}{\sqrt{2\pi}\sigma_s}\frac{1}{\ceil{\frac{1}{2\sqrt{\Deltatarget}}}\Deltatarget} + \max\bracket{\frac{1}{\sigma_s^2}\frac{\mu_i-m_s}{\ceil{\frac{1}{2\sqrt{\Deltatarget}}}\Deltatarget},0}} \diff \mu_i \tag{$F_k (\cdot) \leq 1$}\\
    &= \frac{1}{\sqrt{2\pi}\sigma_s}\frac{1}{\ceil{\frac{1}{2\sqrt{\Deltatarget}}}\Deltatarget}+\frac{1}{\sigma_s^2 \ceil{\frac{1}{2\sqrt{\Deltatarget}}}\Deltatarget}\bracket{\int_{\Real} h_i (\mu_i) \max\bracket{\mu_i - m_s,0} \diff \mu_i} \nonumber\\
    &\leq \frac{1}{\sqrt{2\pi}\sigma_s}\frac{1}{\ceil{\frac{1}{2\sqrt{\Deltatarget}}}\Deltatarget}+\frac{1}{\sigma_s^2 \ceil{\frac{1}{2\sqrt{\Deltatarget}}}\Deltatarget}\bracket{\int_{\Real} h_i (\mu_i) |\mu_i - m_s| \diff \mu_i} \nonumber\\
    &\leq \frac{1}{\sqrt{2\pi}\sigma_s}\frac{1}{\ceil{\frac{1}{2\sqrt{\Deltatarget}}}\Deltatarget}+\frac{1}{\sigma_s^2 \ceil{\frac{1}{2\sqrt{\Deltatarget}}}\Deltatarget}\bracket{\int_{\Real} h_i (\mu_i) \paren{|\mu_i - m_i|+|m_i-m_s|} \diff \mu_i} \nonumber\\
    &\leq \frac{1}{\sqrt{2\pi}\sigma_s}\frac{1}{\ceil{\frac{1}{2\sqrt{\Deltatarget}}}\Deltatarget}+\frac{1}{\sigma_s^2 \ceil{\frac{1}{2\sqrt{\Deltatarget}}}\Deltatarget}(|m_i-m_s| + \frac{\sigma_i\sqrt{2}}{\sqrt{\pi}}) \tag{Mean of Half normal distribution is $\frac{\sigma_i\sqrt{2}}{\sqrt{\pi}}$}\\
    &=\frac{1}{\ceil{\frac{1}{2\sqrt{\Deltatarget}}}\Deltatarget}\underbrace{\bracket{\frac{1}{\sqrt{2\pi}\sigma_s}+\frac{1}{\sigma_s^2}(|m_i-m_s| + \frac{\sigma_i\sqrt{2}}{\sqrt{\pi}})}}_{S_{is}(\bH)/2} \leq \frac{S_{is}(\bH)}{\sqrt{\Deltatarget}} = O(\Deltatarget^{-1/2}). \label{eq:(A2) integral}
\end{align}

Therefore, by Eq.\ (\ref{eq:(A1) integral}) and Eq.\ (\ref{eq:(A2) integral}),
\begin{align*}
    (A) &= (A1) + (A2) = \frac{1}{\Deltatarget^2} \cdot \frac{\pi^2}{3} P_{is} + O(\Deltatarget^{-1/2}),
\end{align*}
 and therefore

\begin{align}
    \int_{\mu \in \Theta_{i}} \ET_s (\bmu) \diff \bH (\bmu) \leq& \frac{B \log \Deltatarget^{-1}}{\Deltatarget^2} P_{is}\bracket{\frac{\pi^2}{3}+1} + O({\Deltatarget^{-1/2}}). \label{eqn:Case1 of upper bound}
\end{align}


\paragraph{Case 2: $s=i$} In this case, let
\begin{align*}
    \int_{\mu \in \Theta_{s}} \ET_s (\bmu) \diff \bH(\bmu)&= 
    \sum_{j \neq s} \int_{\bmu \in \Theta_{sj}} \ET_s (\bmu) \diff \bH(\bmu)\\
    &\leq \sum_{j \neq s} \int_{\bmu \in \Theta_{sj}} \max_{i\neq s} (\ET_i (\bmu)) \diff \bH(\bmu) \tag{$N_s(t)$ increases only when $|\EA_t|>1$, so there should be a competitor}\\
    &\leq \sum_{j \neq s} \int_{\bmu \in \Theta_{sj}} \max_{i\neq s} \bracket{R_0(\Delta_i(\bmu), \Deltatarget)} \diff \bH(\bmu) \tag{Lemma \ref{lem_exp_subopt}}\\
    &\leq \sum_{j \neq s} \int_{\bmu \in \Theta_{sj}} \bracket{R_0(\Delta_j(\bmu), \Deltatarget)} \diff \bH(\bmu) 
\end{align*}

and by the same calculation {as} Case 1, we obtain:

\begin{align*}
    \int_{\mu \in \Theta_{sj}} \bracket{R_0 (\Delta_j(\bmu), \Deltatarget)} \diff \bmu \leq \frac{B \log \Deltatarget^{-1}}{\Deltatarget^2} P_{sj} \bracket{ \frac{\pi^2}{3}+1} + O(\Deltatarget^{-1/2}),
\end{align*}

and therefore

\begin{align}
 \int_{\mu \in \Theta_{s}} \ET_s (\bmu) \diff \bH(\bmu) \leq \frac{B \log \Deltatarget^{-1}}{\Deltatarget^2} \bracket{ \frac{\pi^2}{3}+1} \sum_{j\neq s} P_{sj} + O(\Deltatarget^{-1/2}).\label{eqn:Case2 of upper bound}
\end{align}



For notational convenience, let $B_0 = B \cdot \bracket{\frac{\pi^2}{3}+1}$. From Eq. \eqref{eqn:Sum of Case1 and Case2}, Eq. \eqref{eqn:Case1 of upper bound}, Eq. \eqref{eqn:Case2 of upper bound}, we get

\begin{align}
    \Ex[N_s (T_0)] &= \Ex_{\bmu\sim\bH}\bracket{\Ex_{\bmu}[N_s (T_0)\Ind[\EX(\bmu)]]} +\Ex_{\bmu\sim\bH}\bracket{\Ex_{\bmu}[N_s (T_0)\Ind[\EX(\bmu)^c]]}\\
    &\leq \sum_{i=1}^k \int_{\mu \in \Theta_i} \ET_s (\bmu) \diff \bH(\bmu) + \Ex_{\bmu\sim\bH}\bracket{\Ex_{\bmu}[N_s (T_0)\Ind[\EX(\bmu)^c]]} \tag{Eq. \eqref{eqn:Sum of Case1 and Case2}}\\
    &\leq \sum_{i\neq s} \int_{\mu \in \Theta_i} \ET_s (\bmu) \diff \bH(\bmu)+\int_{\mu \in \Theta_s} \ET_s (\bmu) \diff \bH(\bmu) + \Ex_{\bmu\sim\bH}\bracket{\Ex_{\bmu}[N_s (T_0)\Ind[\EX(\bmu)^c]]} \nonumber\\
    &\leq \frac{B_0 \log \Deltatarget^{-1}}{\Deltatarget^2} \bracket{\sum_{i\neq s} P_{is} + \sum_{j\neq s} P_{sj}} + O(\Deltatarget^{-1/2})+\Ex_{\bmu\sim\bH}\bracket{\Ex_{\bmu}[N_s (T_0)\Ind[\EX(\bmu)^c]]} \tag{Eq. \eqref{eqn:Case1 of upper bound} and \eqref{eqn:Case2 of upper bound}}\\
    &= \frac{B_0 \log \Deltatarget^{-1}}{\Deltatarget^2} \bracket{\Prob(i^*(\bmu) \neq s, \mu_{i^*(\bmu)} - \mu_s \leq \Deltatarget) + \Prob(i^*(\bmu) = s, \mu_s - \mu_{j^*(\bmu)} \leq \Deltatarget)} \nonumber\\
    &+ O(\Deltatarget^{-1/2}) +\Ex_{\bmu\sim\bH}\bracket{\Ex_{\bmu}[N_s (T_0)\Ind[\EX(\bmu)^c]]}.\label{eq: final expected number of pulls for arm s}
\end{align}

Now, the total stopping time is bounded as follows:

\begin{align*}
    \Ex[\tau] &= \Ex[\sum_{s=1}^k N_s(T_0)]\\
    &=\frac{B_0 \log \Deltatarget^{-1}}{\Deltatarget^2} \bracket{\underbrace{\sum_{s=1}^k\Prob(i^*(\bmu) \neq s, \mu_{i^*(\bmu)} - \mu_s \leq \Deltatarget)}_{PSum1} + \underbrace{\sum_{s=1}^k \Prob(i^*(\bmu) = s, \mu_s - \mu_{j^*(\bmu)} \leq \Deltatarget)}_{PSum2}}\\
    &+ O(\Deltatarget^{-1/2})+\Ex[\tau\Ind[\EX^c]].\tag{Eq. \eqref{eq: final expected number of pulls for arm s}}
\end{align*}

The final task we have left is bounding $(PSum1)$ and $(PSum2)$. Let us define $k^*(\bmu)$ as the third best arm in $\bmu$. For $(PSum1)$, 
\begin{align*}
    (PSum1) &= \sum_{s=1}^k\Prob(i^*(\bmu) \neq s, \mu_{i^*(\bmu)} - \mu_s \leq \Deltatarget)\\
    &= \sum_{s=1}^k\Prob(j^*(\bmu) = s, \mu_{i^*(\bmu)} - \mu_{s} \leq \Deltatarget) + \sum_{s=1}^k\Prob(j^*(\bmu) \neq s, \mu_{i^*(\bmu)} - \mu_{s} \leq \Deltatarget)\\
    &\leq \Prob(\mu_{i^*(\bmu)}-\mu_{j^*(\bmu)}\leq \Deltatarget) + \sum_{s=1}^k\Prob(\mu_{i^*(\bmu)} - \mu_{k^*(\bmu)} \leq \Deltatarget) \tag{When $j^*(\bmu)\neq s$, $\mu_{k^*(\bmu)}\geq s$}\\
    &\leq \frac{\delta}{2}+k\cdot \Prob(\mu_{i^*(\bmu)} - \mu_{k^*(\bmu)} \leq \Deltatarget) \tag{Definition of $\Deltatarget$}\\
    &\leq \frac{\delta}{2} + O(\Deltatarget^2). \tag{Lemma \ref{lem:more than two deltatarget}}
\end{align*}
For $(PSum2)$,
\begin{align*}
    (PSum2) &= \sum_{s=1}^k \Prob(i^*(\bmu) = s, \mu_s - \mu_{j^*(\bmu)} \leq \Deltatarget) \\
    &= \Prob(\mu_{i^*(\bmu)} - \mu_{j^*(\bmu)} \leq \Deltatarget)\\
    &=\Prob(\Delta(\bmu) \leq \Deltatarget) \leq \frac{\delta}{2} \tag{Definition of $\Deltatarget$}
\end{align*}
and finally we can conclude
\begin{align}
    \Ex[\tau] &\leq \frac{B_0 \log \Deltatarget^{-1}}{\Deltatarget^2} \delta +O(\Deltatarget^{-1/2})+\Ex[\tau \Ind[\EX^c]]\tag{Eq. \eqref{eq: final expected number of pulls for arm s}, (PSum1) and (PSum2)}\\
    &\leq \frac{B_0 \log \Deltatarget^{-1}}{\Deltatarget^2} \delta +O(\Deltatarget^{-1/2})+\frac{KB \log \Deltatarget^{-1}}{\Deltatarget^2}\cdot \delta^2 \tag{Eq.\eqref{appeqn: expected stopping time minor part}} \\
    &=\frac{B_0 \log \Deltatarget^{-1}}{\Deltatarget^2} \delta +O(\Deltatarget^{-1/2}),
    \label{appeqn: case when thr is bigger than target}
\end{align}
and the proof is completed. 
\end{proof}

{\subsection{Bound that holds for any $\delta$}\label{appsec: nonasymptotic}}

In this case, we need to change the definition of $\Deltatarget$ as 
\begin{equation}
    \Deltatarget:=\min \paren{\max \left\{\Delta\in (0,1): L(\bH, \Delta)\Delta \leq \frac{\delta}{2} \right\}, \min_{i \neq j} (\xi_i L_{ij}(\bH))^2}.
\end{equation}

{Assume that $\Deltatarget<\DeltaThr$.}
Then, all the proof flows of Section \ref{appsec: upper bound proof} after Eq. \eqref{appeqn: definition of Deltatarget} follows accordingly, and we obtain the following upper bound:
\begin{equation}
\label{ineq_nonasym}
    \Ex[\tau] \leq \underbrace{\frac{B_0 \log \Deltatarget^{-1}}{\Deltatarget^2}\delta + \frac{2\sum_{i\neq j} S_{ij}(\bH)}{\sqrt{\Deltatarget}}}_{\eqref{appeqn: case when thr is bigger than target}} +\underbrace{\frac{KB_0 \log \Deltatarget^{-1}}{\Deltatarget^2}\delta^2}_{\eqref{appeqn: expected stopping time minor part}} + \underbrace{B_0 \paren{\sum_{i\neq j \neq k} Q_{ijk}}\log \Deltatarget^{-1}}_{\text{From }(PSum1)} + k^2 \cdot \frac{B_0 \log \DeltaThr^{-1}}{\DeltaThr^2},
\end{equation}
where $S_{ij}$ and $Q_{ijk}$ are defined in Eq. \eqref{eq:(A2) integral} and Lemma \ref{lem:more than two deltatarget}, respectively.

{Eq.~\eqref{ineq_nonasym} also holds for the case of $\Deltatarget \ge \DeltaThr$. In this case, from the definition of $R_0$, we have}
$$\mathcal{T}_s (\bmu) = \Ex_{\bmu}\bracket{R_0(\max\paren{\Delta_s, \Delta_0)} \Ind[\EX]} = \Ex_{\bmu}\bracket{R_0(\DeltaThr) \Ind[\EX]} \leq R_0(\DeltaThr) \Prob_{\bmu}[\EX]= R_0(\DeltaThr),$$ {which implies} \begin{align*}
    \Ex[\tau] &\le k\cdot R_0 (\DeltaThr) + T_0 \cdot \deltatwo \\
    &= k \cdot R_0 (\DeltaThr) \cdot (1+\deltatwo) \tag{$T_0 = k\cdot R_0(\Deltatarget)=R_0(\DeltaThr)$}\\
    &=k \cdot \frac{B_0 \log \DeltaThr^{-1}}{\DeltaThr^2}\cdot (1+ \deltatwo) \\
    &\le 2k \frac{B_0 \log \DeltaThr^{-1}}{\DeltaThr^2} \le \text{Eq.~\eqref{ineq_nonasym}}.
\end{align*} 

In summary, we obtain Eq.~\eqref{ineq_nonasym}. 
%

\subsection{Proof of Lemmas}

\subsubsection{Proof of Lemma \ref{lem:stratified volume}}

\begin{lem}\label{lem:stratified volume}
    For $S+1 \leq \frac{1}{\sqrt{\delta}}$ and for $\delta \leq \paren{\xi_i {L_{is}(\bH)}}^2$, we have 
$$\Prob (i^*(\bmu)=i, \mu_i - \mu_s \in [S\delta, (S+1)\delta] )\leq 2 \Prob (i^*(\bmu)=i, \mu_i - \mu_s \leq \delta).$$
\end{lem}

\begin{proof} We have
\begin{align*}
\int_{\Theta_{i}} \Ind\left[|\mu_i-\mu_{s}| \in [S\delta, (S+1)\delta] \right] \diff \bH(\bmu)
&= 
\int_{\Theta_{\backslash i}}
\int_{\mu_i = \mu_{s}+S\delta}^{\mu_{s} + (S+1)\delta}
h_i(\mu_i)
\diff \mu_i
\diff \bH_{\setminus i}(\bmu_{\setminus i})\\
&\le 
\int_{\Theta_{\backslash i}}
\delta \bracket{h_i(\mu_{s}) + \frac{e^{-1/2}}{\xi_i}(S+1)\delta}
\diff \bH_{\setminus i}(\bmu_{\setminus i}) \tag{by Lipschitz property of Gaussian, $e^{-1/2}/\xi_i$ is the steepest slope of $N(m_i, \xi_i^2)$}\\
&=
\delta (\underbrace{\bracket{\int_{\Theta_{\backslash i}}
h_i(\mu_{s})
\diff \bH(\bmu_{\setminus i})}}_{L_{is}(\bH)}+ \frac{e^{-1/2}}{\xi_i} (S+1)\delta) \\
&\leq (L_{is} (\bH) +  \frac{1}{\xi_i} (S+1)\delta) \delta.
\end{align*}

When $S+1 < \frac{1}{\sqrt{\delta}}$ and $\sqrt{\delta} < {L_{is} (\bH)}{\xi_i}$, we have

$$\int_{\Theta_{i}} \Ind\left[|\mu_i-\mu_{j}| \in [S\delta, (S+1)\delta] \right] \diff \bH(\bmu) \leq 2L_{is} (\bH) \delta
$$
as intended. 
\end{proof}

\subsubsection{Proof of Lemma \ref{lem:more than two deltatarget}}

We will show that the probability that three or more arms are $\delta$-close is $O(\delta^2)$. Namely: 
\begin{lem}\label{lem:more than two deltatarget}
   We have $\Prob(\mu_{i^*(\bmu)} - \mu_{k^* (\bmu)} \leq \Deltatarget) = O(\delta^2)$.
\end{lem}

\begin{proof}

For any $i \ne j \ne k \in [k]$, we have
\begin{align*}\label{ineq_threenear}
\lefteqn{
\int \Ind\left[
|\mu_i-\mu_j|, |\mu_i - \mu_k| \le \delta
\right] \diff \bH(\bmu)
}\\
&=
\int_{\bmu_{\setminus jk}} \int_{\mu_j= \mu_i-\delta}^{\mu_i+\delta} \int_{\mu_k = \mu_i-\delta}^{\mu_i+\delta}
h_{j}(\mu_j) h_k (\mu_k) \diff \mu_k \diff \mu_j
\diff \bH_{\setminus jk}(\bmu_{\setminus jk})\\
&\le 
\int_{\bmu_{\setminus jk}} \int_{\mu_j= \mu_i-\delta}^{\mu_i+\delta} \int_{\mu_k = \mu_i-\delta}^{\mu_i+\delta}
\paren{h_{j}(\mu_i)+\frac{e^{-1/2}\delta}{\xi_j}} \paren{h_k (\mu_i)+\frac{e^{-1/2}\delta}{\xi_k}}\diff\mu_k \diff\mu_j
\diff\bH_{\setminus jk}(\bmu_{\setminus jk})\tag{Lipschitz property of Gaussian}
\\
&\le 
(2\delta)^2 \underbrace{\int_{\bmu_{\setminus jk}} \bracket{h_{j}(\mu_i) h_k (\mu_i)+ O(\delta)} \diff \bH_{\setminus jk}(\bmu_{\setminus jk})}_{=:Q_{ijk}(\bH)} = O(\delta^2).
\end{align*}
\end{proof}

\section{Experimental details}\label{appsec:experiment details}

\subsection{Stopping condition} 

\paragraph{TTTS} We use our theoretical results stated in Section \ref{sec:main algorithm} for our stopping criterion. For TTTS, we use Chernoff's stopping rule, as \cite{garivier16, jourdan2022top} did. Here is the description how it works: for each arm $i, j \in [k]$, let 
$$ \hat{\mu}_{ij} (t):= \frac{N_i (t)}{N_i(t) + N_j (t)} \hat{\mu}_i (t)+ \frac{N_j (t)}{N_i(t) + N_j (t)}  \hat{\mu}_j (t),$$
and define
$$ Z_{ij} (t) := N_i (t) \cdot KL_i(\hat{\mu}_i, \hat{\mu}_{ij}) + N_j (t) \cdot KL_j(\hat{\mu}_j, \hat{\mu}_{ij}).$$
Now the stopping time is defined as:
$$\tau_{TTTS} := \inf \{t\in\mathbb{N}: \max_{a \in [k]} \min_{b \in [k]\backslash a} Z_{ab}(t) \geq \beta(t, \delta)\}$$
for some threshold function $\beta(t, \delta)$, which is defined by the following proposition of \cite{garivier16}:

\begin{thm}[\citealt{garivier16}, Proposition 12]
    Let $\bmu$ be an exponential family bandit model. Let $\delta\in (0, 1)$ and $\alpha > 1$. There exists a constant $C = C(\alpha, k)$ such that whatever the sampling strategy, using Chernoff’s stopping rule with the threshold
    $$ \beta(t,\delta) = \log \frac{C t^\alpha}{\delta}$$
    ensures that for all $\bmu$, $\Prob_{\bmu} \paren{\tau < \infty, J \neq i^* (\bmu)} \leq \delta$.
\end{thm}

In this theorem, we give an advantage to the stopping time of TTTS by setting $\alpha=1, C=1$. Theoretically, $C(\alpha,k)>1$ and $C \to \infty$ as $\alpha \to 1_+$, but we set the threshold smaller than the theoretical guarantee so that TTTS stops earlier. 

\paragraph{TTUCB} We followed the stopping rule of the original paper \cite{jourdan2022non}. Let $\EC_G(x):= \min_{\lambda \in (\frac{1}{2},1]} \frac{2\lambda -2\lambda\log(4\lambda) +\log \zeta(2\lambda) -0.5\log (1-\lambda)+x}{\lambda}$ where $\zeta$ is a Riemann $\zeta$ function, and 
$$c(n, \delta):= 2\EC_G(\frac{1}{2} \log \frac{k-1}{\delta}) + 4 \log (4+\log \frac{n}{2}).$$
Let $\hat{i}_t:= \arg \max_{i \in [k]} \hat{\mu}_i (t)$, the empirical best arm at step $t$. The TTUCB algorithm stops when 

$$ \min_{i\neq \hat{i}_t} \frac{\hat{\mu}_{\hat{i}}(t)-\hat{\mu}_{i}(t)}{\sqrt{\frac{1}{N_{hat{i}}(t)} + \frac{1}{N_i(t)}}} \geq \sqrt{c(t, \delta)}.$$

When the algorithm stops sampling, the TTUCB algorithm recommends the empirical best arm as its final suggestion. 

Since the computation of $\EC_G (x)$ involves optimization, it is computationally heavy when the number of samples is excessively large (as our Table \ref{table:simulation result}). Instead, we approximated $\EC_G (x) \approx x + \log x$ as mentioned in \cite{jourdan2022non}. 

\subsection{NoElim algorithm} \label{appsubsec:noelim}

The NoElim algorithm is shown in Algorithm \ref{alg:noelim}.

\begin{algorithm}[h]
\caption{No Elimination (NoElim) Algorithm} 
\label{alg:noelim}
\begin{algorithmic}
\STATE {\bfseries Input:}{ Confidence level $\delta$, prior $\bH$}
\STATE $\Deltatarget:=\frac{\delta}{4L(\bH)}$
\STATE Initialize the candidate of best arms $\EA (1) = [k]$ 
\STATE $t=1$
\WHILE {True}
\STATE \textit{\textbf{Draw each arm in $[k]$ once. }} \COMMENT{Main Difference with Algorithm \ref{alg:elim}}.
\STATE $t \to t+|\EA(t)|$
$\hat{\Delta}^{\mathrm{safe}}(t)$.
\FOR{$i \in \EA(t)$}
\STATE Calculate $\rmUCB(i,t)$ and $\rmLCB(i,t)$ from \eqref{eqn:UCB LCB bound}. 
\IF{$\rmUCB(i,t) \le \max_j \rmLCB(j,t)$} 
\STATE $\EA(t) \leftarrow \EA(t) \setminus \{i\}$.
\ENDIF   
\ENDFOR
\IF{$|\EA(t)|=1$}
\STATE {\bfseries Return} arm $J$ in $\EA(t)$.\label{line_stopone} 
\ENDIF
\STATE Calculate safe empirical gap 
\IF{$\hat{\Delta}^{\mathrm{safe}}(t) \le \Deltatarget$}
\STATE {\bfseries Return} arm $J$ which is uniformly sampled from $\EA(t)$.\label{line_stoptwo} 
\ENDIF
\ENDWHILE
\end{algorithmic}
\end{algorithm}

\subsection{Tables including computation time}\label{appsubsec:comp time}
For all tables in this section, Comp represents the average computation time (second). 
\begin{table}[h!]
\caption{Extended version of Table \ref{table:simulation result} with computation time. }
\label{table:simulation result with comp}
\begin{center}
\begin{small}
\begin{sc}
\begin{tabular}{c|cccc}
\toprule
& Avg& Max & Error & Comp.\\
\midrule
Alg. \ref{alg:elim}  &   $1.06 \times 10^4$& $2.35\times 10^5$ & 1.5\% & $0.17$\\
TTTS & $1.56\times 10^5$ & $1.09\times 10^8$ & 0.5\% & 27.6\\
TTUCB & $1.95\times 10^5$ & $1.13\times 10^8$ & 0\% &5.07\\
\bottomrule
\end{tabular}
\end{sc}
\end{small}
\end{center}
\vskip -0.1in
\end{table}

\begin{table}[h!]
\caption{{Extended version of Table \ref{table:without elim result} with computation time.} }
\label{table:without elim result with comp}
\begin{center}
\begin{small}
\begin{sc}
\begin{tabular}{c|cccc}
\toprule
& Avg& Max & Error & Comp.\\
\midrule
Alg. \ref{alg:elim}  &  $2.69\times 10^5$& $1.66\times 10^7$ & 0.6\% & $1.59$\\
NoElim & $1.29\times 10^6$ & $8.25\times 10^7$ & 0\% & 5.5\\
\bottomrule
\end{tabular}
\end{sc}
\end{small}
\end{center}
\vskip -0.1in
\end{table}

\subsection{Miscellaneous}

\paragraph{Computation of $\Delta_0$} From Definition \ref{def:Lij}, 
$$ L_{ij}(\bH) := \int_{-\infty}^{\infty} h_i (x) h_j (x) \prod_{s: s \in [k]\backslash \{i,j\}} H_s (x) \diff x.$$
In \textbf{Scipy} package, there are functions for computing the cumulative function of Gaussian $H_s$ (scipy.norm.cdf) and $h_i$ (scipy.norm.pdf). \textbf{Scipy} package also supports the numerical integration (scipy.integrate.quad) which we use to numerically compute $L_{ij}$ in our experiments. 

\paragraph{Codes} The codes are in the supplementary material and will be published in public through GitHub when this paper is accepted.

\paragraph{Hardware} We used Python 3.7 as our programming language and Macbook Pro M2 16 inch as our hardware.

\section{Sufficiently small $\delta$}\label{appsec:delta0}

The second result of the Lemma \ref{lem:volume_rewrite} is used for the lower bound.
For this result to hold, we need the following two conditions for $D_1$:
\begin{itemize}
    \item Proof of Lemma \ref{lem:volume_lem}, second result: For the proof of Eq. \eqref{appeqn: volume lemma lower}, we used $|m_i|\leq \frac{1}{2\sqrt{\Delta}}$
    \item Proof of Lemma \ref{lem:volume_lem}, second result: For the proof of Eq. \eqref{appeqn: volume lemma lower}, we used $\frac{1}{\xi_i} \Delta^2 > 2\exp \paren{-\frac{1}{8\Delta \xi_i^2}}+2\exp \paren{-\frac{1}{8\Delta \xi_j^2}}$. $\Delta < D_0 (\bH)$ where 
    \begin{align*}D_0 (\bH) := \begin{cases}
        W(-\frac{1}{32 \max_{i\in [k]} \xi_i^{3/2}})& \text{If $\max_{i\in[k]} \xi_i > \sqrt[3]{\frac{e^2}{2^{10}}}$} \\
        1 & \text{Otherwise}
    \end{cases}
    \end{align*}
    satisfies the condition. Here $W$ is the Lambert W function with the principal branch.
\end{itemize}

For the lower bound proof, we consider sufficiently small $\delta$ subject to the following constraints on 
$\tilde\Delta = \frac{32e^4}{L(\bH)}\delta$: 
\begin{itemize}
    \item Proof of Lemma \ref{lem:volume_lem}, second result: For the proof of Eq. \eqref{appeqn: volume lemma lower}, we used $|m_i|\leq \frac{1}{2\sqrt{\tilde\Delta}}$.
    \item Proof of Lemma \ref{lem:volume_lem}, second result: For the proof of Eq. \eqref{appeqn: volume lemma lower}, we used $\frac{1}{\xi_i} \tilde\Delta^2 > 2\exp \paren{-\frac{1}{8\tilde\Delta \xi_i^2}}+2\exp \paren{-\frac{1}{8\tilde\Delta \xi_j^2}}$. $\tilde\Delta < D_0 (\bH)$ will satisfy this condition. 
    \item For the Lemma \ref{lem:ratiolemma}:
    $\tilde\Delta<D_1 (\bH):=\min_{i \neq j} \bracket{\left|\frac{m_i}{\sigma_i^2} -\frac{m_j}{\sigma_j^2}\right|^{-1}, \bracket{\frac{1}{2\sigma_i^2}+\frac{1}{2\sigma_j^2}}^{-2}}$.
    \item To make $L'(\bH, \tilde\Delta)\tilde\Delta \in (\frac{1}{2} L(\bH), 2L(\bH))$, $\tilde\Delta \leq \frac{L(\bH)}{4{\sum_{i\in[k]}\frac{k-1}{\xi_i}}}$
\end{itemize}

In summary, Theorem 1 holds for any 
\[
\delta \le \delta_L := \frac{L(\bH)}{32e^4} \cdot \min \paren{D_0(\bH), D_1(\bH), \min_{i \in [k]}\frac{1}{4m_i^2}, \frac{L(\bH)}{4(k-1){\sum_{i\in[k]}\frac{1}{\xi_i}}}}.
\]

For the upper bound proof (Theorem \ref{thm: elim upperbound}), we consider $\delta$ such that $\Deltatarget = \frac{\delta}{4L(\bH)}$ satisfies the following conditions: 
\begin{itemize}
    \item $\Deltatarget < \frac{L(\bH)}{{\sum_{i\in[k]}\frac{k-1}{\xi_i}}}$ to make $L(\bH, \Deltatarget)\cdot \Deltatarget \leq 2L(\bH) \Deltatarget$ by the first result of Lemma \ref{lem:volume_rewrite}.
    \item $\Deltatarget \leq \min_{i,j\in [k], i\neq j} ({L_{ij}}{\xi_i})^2$ for the proof and usage of Lemma \ref{lem:stratified volume}, and
    \item $\Deltatarget < \DeltaThr$.
\end{itemize}



\newpage

\end{document}